\documentclass[twoside]{article}

\usepackage{aistats2020}
\usepackage[utf8]{inputenc} 
\usepackage[T1]{fontenc}    
\usepackage{hyperref}       
\usepackage{url}            
\usepackage{booktabs}       
\usepackage{amsfonts}       
\usepackage{nicefrac}       
\usepackage{microtype}      
\usepackage{amsmath,amsthm,amssymb}
\usepackage{bm}
\usepackage{color}
\usepackage{dsfont}
\usepackage{blkarray, bigstrut}
\usepackage{multicol}
\usepackage{graphicx}
\usepackage{makecell}
\usepackage{wrapfig}
\usepackage{enumitem}
\usepackage{booktabs}
\usepackage{algorithm}
\usepackage{algorithmic}
\usepackage{multirow}
\usepackage{subcaption}
\usepackage{bbold}
\usepackage{caption}

\newtheorem{corollary}{Corollary}
\newtheorem{theorem}{Theorem}
\newtheorem{lemma}{Lemma}
\def\argmax{\mathop{\rm argmax}}

\def\byi{\boldsymbol{y_i}}
\def\sign{\mathop{\rm sign}}
\def\diag{\mathop{\rm diag}}

\definecolor{darkorange1}{rgb}{1.0, 0.8, 0.0} \newcommand\resa[1]{\textcolor{darkorange1}{\textbf{#1}}}
\definecolor{darkorange2}{rgb}{1.0, 0.67, 0.0} \newcommand\resb[1]{\textcolor{darkorange2}{\textbf{#1}}}
\definecolor{darkorange3}{rgb}{1.0, 0.54, 0.0} \newcommand\resc[1]{\textcolor{darkorange3}{\textbf{#1}}}
\definecolor{darkorange4}{rgb}{1.0, 0.41, 0.0} \newcommand\resd[1]{\textcolor{darkorange4}{\textbf{#1}}}
\definecolor{darkorange5}{rgb}{1.0, 0.28, 0.0} \newcommand\rese[1]{\textcolor{darkorange5}{\textbf{#1}}}

\definecolor{Red}{rgb}{1.0,0.0,0.0}

\begin{document}

%

%

\twocolumn[

\aistatstitle{Adversarial Robustness via Label-Smoothing}

\aistatsauthor{ Morgane Goibert\\m.goibert@criteo.com \And Elvis Dohmatob\\e.dohmatob@criteo.com}

\aistatsaddress{ Criteo AI Lab \And  Criteo AI Lab} ]

\begin{abstract}
We study Label-Smoothing as a means for improving adversarial robustness of supervised deep-learning models. After establishing a thorough and unified framework, we propose several variations to this general method: adversarial, Boltzmann and second-best Label-Smoothing methods, and we explain how to construct your own one. On various datasets (MNIST, CIFAR10, SVHN) and models (linear models, MLPs, LeNet, ResNet), we show that Label-Smoothing in general improves adversarial robustness against a variety of attacks (FGSM, BIM, DeepFool, Carlini-Wagner) by better taking account of the dataset geometry. The proposed Label-Smoothing methods have two main advantages: they can be implemented as a modified cross-entropy loss, thus do not require any modifications of the network architecture nor do they lead to increased training times, and they improve both standard and adversarial accuracy.
\end{abstract}

\section{INTRODUCTION}
\label{intro}

Neural Networks (NNs) have proved their efficiency in solving classification problems
in areas such as computer vision \cite{krizhevsky2012imagenet}. Despite these
successes, recent works have shown that NN
are sensitive to adversarial examples (e.g \cite{szegedy2013intriguing}), which is problematic for critical
applications \cite{sitawarin2018darts} like self-driving cars and medical diagnosis.
Many strategies have thus been developed to improve robustness and in an "arms race", different
attacks have been proposed to evade these defenses. Broadly speaking, an
adversarial attack succeeds when an image is slightly modified so that it still looks \emph{to a human} like it belongs to a
right class, but a classifier misclassifies it. Despite the number of
works on it \cite{goodfellow2014explaining, fawzi2016robustness,
  tanay2016boundary, tramer2017space}, there is still no complete understanding of the adversarial
phenomenon. Yet, the vulnerability of NN to
adversarial attacks suggests a shortcoming in the generalization of the network. As overconfidence in predictions hinders generalization, addressing it can be a good
way to tackle adversarial attacks \cite{zheng2018improvement}. Label-Smoothing
(LS) ~\cite{labelsmoothing,labelsmoothingbis} is a method which creates uncertainty in the labels of a dataset used to train a
NN. This uncertainty helps to tackle the over-fitting issue, and thus LS can be
an efficient method to address the adversarial attack phenomenon.

\subsection{Notations And Terminology}
\paragraph{General} We denote by $\mathcal{X}$ the input space of dimension $d$, and $\mathcal{Y} = \{1,...,K\}$ the label space, $P_{X,Y}$ is the true (unknown) joint distribution of $(X,Y)$ on $\mathcal{X} \times \mathcal{Y}$. $\Delta_K := \{q \in \mathbb R^K \mid q \ge 0,\; \sum_{k=1}^K q_k = 1\}$ is the $(K-1)$-dimensional probability simplex $\Delta_K$, identified with the set $\mathcal P(\mathcal Y)$ of probability distributions on $\mathcal Y$.

An iid sample drawn from $P_{X,Y}$ is written $S_n = \{(x_1, y_1), ... (x_n,
y_n)\}$. To avoid any ambiguity with the label $y \in [\![K]\!]$, we use boldface $\boldsymbol{y}=(0,\ldots,1,\ldots,0) \in \Delta_K$ to denote the one-hot encoding of $y$. The empirical distribution of the input-label pairs $(x,y)$ is written $\hat{P}^n_{X,Y}$.
A classifier is a measurable function $h : \mathcal{X} \to \mathcal{Y}$, usually parametrized by real parameters $\theta$; for example NN with several hidden layers, with the last always being a softmax function. The logits of the classifier (pre-softmax) are written $z(x,\theta)$, and $z^{(k)}(x,\theta)$ is its component for the $k$th class. The prediction vector of the classifier (post-softmax) is written $p(x;\theta)$. 

\paragraph{Adversarial Attacks}
An attacker constructs an \emph{adversarial example} based on a clean input $x$ by adding a perturbation to it: $x^{adv} = x + \delta$. The goal of the attack is to have $h(x^{adv}) \neq h(x)$. The norm of the the perturbation vector $\delta$ measures the size of the attack. In this work, we limit ourselves to $\ell_\infty$-norm attacks, wherein $\|\delta\|_\infty := \max_{k=1}^p|\delta^k|$. A tolerance threshold $\varepsilon$ controls the size of the attack: the attacker is only allowed to inflict perturbations of size $\|\delta\|_\infty \le \varepsilon$.

\subsection{Related Works}


Our work will focus on untargeted, white-box attacks, i.e. threat models that only seek to fool the NN (as opposed to tricking it into predicting a specific class), and have unlimited access to the NN parameters. State-of-the art attacks include FGSM \cite{goodfellow2014explaining} a very simple, fast and popular attack; BIM \cite{kurakin2016adversarial}, an iterative attack based on FGSM; DeepFool \cite{moosavi2016deepfool} and C\&W \cite{carlini2017towards}. Note that the tolerance threshold $\varepsilon$ can be explicitely tuned in FGSM and BIM, but not in DeepFool and C\&W. 

Many defences have been proposed to counteract adversarial attacks. The main one is adversarial training \cite{goodfellow2014explaining, madry2017towards}, which consists in feeding a NN with both clean and adversarially-crafted data during training. This defense method will be used in this paper as a baseline for comparative purposes. Another important method is defensive distillation \cite{papernot2016distillation, papernot2016effectiveness}, which is closely related to LS. This method trains a separate NN algorithm and uses its outputs as the input labels for the main NN algorithm. This was an efficient defense method until recently broken by C\&W attack \cite{carlini2017towards}. These defense methods are both very efficient, but time-costly and also, can reduce standard accuracy \cite{tsipras2018robustness}.

LS was first introduced as a regularization method \cite{pereyra2017regularizing, labelsmoothing}, but was also briefly studied as a defense method in \cite{shafahi2018label}. One of the contributions of our paper is to generalize the idea of LS proposed and used in these previous works, and  propose three novel variants relevant for the adversarial issue. We develop theoretical as well as empirical results about the defensive potential of LS.

For a more thorough introduction to the field, interested readers can refer to surveys like \cite{akhtar2018threat, zhang2018adversarial}.

\subsection{Overview of main contributions}
In section \ref{contrib}, we develop a unified framework for Label-Smoothing (LS) of which \cite{labelsmoothing} is a special case. We propose a variety of new LS methods, the main one being Adversarial Label-Smoothing (ALS). We show that in general, LS methods induce some kind of logit-squeezing, which results in more robustness to adversarial attacks. In section \ref{sec:understanding}, we give a complete mathematical treatment of the effect of LS in a simple case, with regards to robustness to adversarial attacks. Section \ref{exp} reports empirical results on real datasets. In section \ref{conclu}, we conclude and provide ideas for future works.

Proofs of all theorems and lemmas are provided in the Appendix (supplementary materials).

\section{UNIFIED FRAMEWORK FOR LS}
\label{contrib}

In standard classification datasets, each example $x$ is hard-labeled with exactly one class $y$. Such overconfidence in the labels can lure a classification algorithm into over-fitting the input distribution ~\cite{labelsmoothing}.
LS ~\cite{labelsmoothing,labelsmoothingbis} is a resampling technique wherein one replaces the vector of probability one on the true class $\boldsymbol{y}$ (i.e the one-hot encoding) with a different vector $q$ which is "close" to  $\boldsymbol{y}$.
Precisely, LS withdraws a fraction of probability mass from the "true" class label and reallocates it to other classes. As we will see, the redistribution method is quite flexible, and leads to different LS methods.

Let $\text{TV}(q'\|q):= (1/2)\|q'-q\|_1$ be the Total-Variation distance between two probability vectors $q',q \in \Delta_K$.
For $\alpha \in [0, 1]$, define the uncertainty set of acceptable label distributions $\mathcal{U}_\alpha(\hat{P}^n_{X,Y})$ by
\begin{equation*}
  \begin{split}
    &\mathcal{U}_\alpha(\hat{P}^n_{X,Y}) := \{\hat{P}^n_{X}\hat{Q}^n_{Y|X} \mid \operatorname{TV}(\hat{Q}^n_{Y|x}\|\hat{P}^n_{Y|x}) \le \alpha, \\
    & \quad\quad\quad\quad\quad\quad\;\forall x \in \mathcal X\}\\
  &=\left\{\frac{1}{n}\sum_{i=1}^n\delta_{x_i} \otimes q_i \; | \; q_i \in \Delta_K,\; \text{TV}(q_i\|\delta_{y_i})\le \alpha,\right. \\
  & \quad\quad\;\forall i \in [\![n]\!] \Bigg\},
  \end{split}
\end{equation*}
made up of joint distributions $\hat{Q}^n_{X,Y} \in \mathcal{P}(\mathcal X \times \mathcal Y)$
on the dataset $S_n$, for which the conditional label distribution
$q_i:= \hat{Q}^n_{Y|X=x_i} \in \Delta_K$ is within  TV distance less than $\alpha$ of the one-hot encoding of the observed label ${y_i}$.
By direct computation, one has that $\text{TV}(q_i\|{y_i}) = (1/2)
\left(\sum_{j \ne y_i}q_i^{(j)} + 1 - q_i^{(y_i)}\right) = 1-q_i^{(y_i)}$ and so the uncertainty set can be rewritten as
\begin{eqnarray*}
  \begin{split}
  \mathcal U_\alpha(\hat{P}^n_{X,Y}) &=\left\{\frac{1}{n}\sum_{i=1}^n \delta_{x_i} \otimes q_i \; | \;  q_i \in \Delta_K, \right. \\
  & \;q_i^{(y_i)}\ge 1 - \alpha \; \forall i \in [\![1,n]\!]\Bigg\}.
  \end{split}
\end{eqnarray*}
Any conditional label distribution $q_i$ from the uncertainty set $\mathcal U_\alpha(\hat{P}^n_{X,Y})$ can be written
\begin{eqnarray}
  q_i = (1-\alpha)\boldsymbol{y_i}+\alpha q_i', \quad q_i' \in \Delta_K.
  \label{eq:general_q}
\end{eqnarray}
Different choices of the probability vector $q_i'$ (which as we shall see in section ~\ref{subsec:als}, can depend on the model parameters!) lead to different Label-Smoothing methods. The training of a NN with general label smoothing then corresponds to the following optimization problem:
\begin{eqnarray}
\label{eq:LS_program}
    \min_{\theta} \frac{1}{n}\sum_{i=1}^n \operatorname{SmoothCE}(x_i, q_i;\theta),
\end{eqnarray}
where $\operatorname{SmoothCE}(x,q;\theta)$ is the \textit{smoothed} cross-entropy loss (a generalization of the standard cross-entropy loss), defined by
\begin{align*}
    \operatorname{SmoothCE}(x,q; \theta) &:= -q^T \log(p(x;\theta))\\
    &=-\sum_{k=1}^Kq^{(k)}\log(p^{(k)}(x;\theta)).
  \label{eq:loss}
\end{align*}

It turns out that the optimization problem \eqref{eq:LS_program} can be rewritten as the optimization of a usual cross-entropy loss, plus a penalty term on the gap between the components of logits (one logit per class) produced by the model on each example $x_i$.

\begin{theorem}[General Label-Smoothing Formulation]
\label{thm:main_thm}
\textit{The optimization problem \eqref{eq:LS_program} is equivalent to the logit-regularized problem
\begin{eqnarray*}
    \min_{\theta}L_n(\theta)+ \alpha R_n(\theta),
\end{eqnarray*}
where $L_n(\theta) :=
-\frac{1}{n}\sum_{i=1}^n\log(p(x_i;\theta))$ is the standard cross-entropy
    loss, and
    \begin{eqnarray*}
      R_n(\theta) := \frac{1}{n}\sum_{i=1}^n(\boldsymbol{y_i}-q'_i)^T z_i.
      \label{eq:reg_gen}
    \end{eqnarray*}
where $z_i:=z(x_i; \theta) \in \mathbb R^K$ is the logits vector for $x_i$.
}
\end{theorem}

In the next two subsections, we present four different LS methods that are relevant to tackle the adversarial robustness issue. A summary of these methods are presented in Table \ref{Tab:compare_dro}. \\[0cm]
\vspace{-0.3cm}
\begin{table*}[!h]  
  \centering
\begin{tabular}{|c|c|c|c|c|}
\hline
         Paper & Name & $q_i'$ & Induced logit penalty $R_n(\theta)$\\
  \hline
  ~\cite{labelsmoothing} & standard label-smoothing (SLS) & $\frac{1-\byi}{K-1}$ & $\sum_{i=1}^n\left(\frac{K\byi-1}{K-1}\right)^T z_i$\\\hline
  Our paper & adversarial label-smoothing (ALS) & $\byi^{\text{worst}}$, see \eqref{eq:als} & $\frac{1}{n}\sum_{i=1}^n z_i^{(y_i'^{\text{worst}})}-z_i^{(y_i)}$\\\hline
  Our paper & Boltzmann label-smoothing (BLS) & ${q_i'}^{\text{Boltz}}$, see \eqref{eq:boltz} & $\frac{1}{n}\sum_{i=1}^n ({q_i'}-\byi)^Tz_i$\\\hline    
  Our paper & second-best label-smoothing (SBLS) & $\byi^{\text{SB}}$, see \eqref{eq:sbls} & $\frac{1}{n}\sum_{i=1}^n z_i^{(y_i'^{\text{SB}})}-z_i^{(y_i)}$\\\hline    
\end{tabular}
\caption{\textbf{Different of LS methods.} They all derive from the general equation \eqref{eq:general_q}.} \vspace{-0.2cm}
\label{Tab:compare_dro}
\end{table*}

\subsection{Adversarial Label-Smoothing}
\label{subsec:als}
\emph{Adversarial Label-Smoothing} (ALS) arises from the worst possible smooth label $q_i$ for each example $x_i$. To this end, consider the two-player game:
\begin{eqnarray}
\label{eq1}
    \min_{\theta}\max_{\hat{Q}^n \in \mathcal U_\alpha(\hat{P}^n_{X,Y})} \frac{1}{n}\sum_{i=1}^n \operatorname{SmoothCE}(x_i, q_i;\theta),
  \label{eq:grand}
\end{eqnarray}
The inner problem in \eqref{eq1} has an analytic solution (see Appendix \ref{app:als}) given by, $\forall \; i \in [\![1,n]\!]$:
\begin{eqnarray}
  q_i = q_i(\theta) = (1-\alpha) \boldsymbol{y_i} +
  \alpha{\boldsymbol{y^{\text{worst}}_i}},
  \label{eq:als}
\end{eqnarray}
where
 $ y_i^{\text{worst}} \in 
 \text{argmin}_{k=1}^K z^{(k)}(x_i,\theta)
 $
is the index of the smallest component of the logits vector $z_i$ for input $x_i$, and $\boldsymbol{y_i}^{\text{worst}}$ is the one-hot encoding thereof.

\paragraph*{Interpretation Of ALS}
The scalar $\alpha \in [0, 1]$ acts as a smoothing parameter:
if $\alpha=0$, then $q_i(\theta)= \boldsymbol{y_i}$,
and we recover hard labels. If $\alpha = 1$, the adversarial weights $q_i(\theta)$ live
in the sub-simplex spanned by the smallest components of the predictions vector
$p(x_i;\theta)$. For $0 < \alpha < 1$, $q_i(\theta)$ is a proper convex combination of the two
previous cases. Applying Theorem \ref{thm:main_thm}, we have:

\begin{corollary}[ALS enforces logit-squeezing]
  The logit-regularized problem equivalent of the ALS problem \eqref{eq:grand} is given by:
    $\displaystyle{\min_{\theta}L_n(\theta)+ \alpha R_n(\theta)}$,
where
$$
      R_n(\theta) :=
      \frac{1}{n}\sum_{i=1}^n(\boldsymbol{y_i}-\boldsymbol{y^{\text{worst}}_i})^T z_i
      = \frac{1}{n}\sum_{i=1}^n
       z_i^{(y_i)} -  z_i^{(y^{\text{worst}}_i)}.
$$
    \label{cor:mals}
\end{corollary}
\vspace{-0.5cm} 
For each data point $x_i$ with true label $y_i$, the logit-squeezing penalty term $R_n(\theta)$ forces the model to refrain from making over-confident predictions,
corresponding to large values of $z_i^{(y_i)}-z_i^{({y_i}^{\text{worst}})}$ that can lead to overfitting. This means that every class label receives a positive prediction output probability: $\forall \; k \in [\![K]\!], \; p_i^{(k)} > 0$. The resulting models are less vulnerable to adversarial perturbations on the input $x$.

One can also see ALS as the label analog of
adversarial training \cite{goodfellow2014explaining, kurakin2016adversarial}.
Instead of modyfing the input data $x$, we modify the label data $y$. However, unlike adversarial training, ALS is attack-independent: it does not require to choose a specific attack
method to be trained on.
Moreover, in contrary to adversarial training, ALS does not entail any significant increase in training time over standard training.

\paragraph*{ALS Implementation}
We noted that ALS only consists in redefining a loss, and using the smoothed cross-entropy with an adversarially-modified version of the one-hot encoding of the class labels. It is very simple to
implement, and computationally as efficient as standard training.

See algorithm \ref{alg:als} for an easy implementation of ALS.

\vspace{-0.2cm}
\begin{algorithm}[H]
   \caption{Adversarial Label-Smoothing (ALS) training}
   \label{alg:als}
\begin{algorithmic}
   \STATE {\bfseries Input:} training data $(x_1, y_1), (x_2,y_2), \ldots$;
   a given model; smoothing parameter $\alpha \in [0, 1]$;
   \FOR{each epoch}
   \FOR{each mini-batch $x=x_{1:m}$, $y=y_{1:m}$}
   \STATE \textbf{Smooth labels} $y_{ALS} \leftarrow (q_1,\ldots,q_m)$ via \eqref{eq:als}
   \STATE \textbf{Get predictions} $y_{pred} \leftarrow model(x)$
   \STATE \textbf{Compute loss} $\leftarrow \text{SmoothCE}(y_{pred}, y_{ALS})$
   \STATE \textbf{Update} model parameters $\theta$ via back-prop
   \ENDFOR
   \ENDFOR
\end{algorithmic}
\end{algorithm}
\vspace{-0.5cm}

\subsection{Other Label-Smoothing Methods}
\label{subsec:other_methods}

\paragraph*{Standard Label-Smoothing} \emph{Standard Label-Smoothing} (SLS) is
the method developed in \cite{labelsmoothing}. It corresponds to uniformly re-assigning the mass $\alpha$ removed from
the real class over the other classes. That is,
the term $q_i' \in \Delta_K$ in \eqref{eq:general_q} is given by
\begin{eqnarray}
q_i' = \sum_{k=1, k \neq y_i}^K \frac{1}{K-1}\boldsymbol{y^{(k)}} = \frac{1}{K-1}(1-\boldsymbol{y_i}).
\label{eq:std}
\end{eqnarray}
In this case,
$R_n(\theta)=\sum_{i=1}^n\left(\frac{K\boldsymbol{y_i}-1}{K-1}\right)^T z_i$. If $K=2$, with a perfect model (i.e. $z_i^{(k)} \ge 0$ if $y_i^{(k)} = 1$ and $z_i^{(k)} < 0$ else),
we have $R_n(\theta)=\sum_{i=1}^n \|z_i\|_1$, an $\ell_1$-norm penalty on the logits.

\paragraph*{Boltzmann Label-Smoothing} ALS puts weights on only two classes: the true class label (due to the constraint of the model) and the class label which minimizes the logit vector.
It thus gives "two-hot" labels rather than "smoothed"
labels. Replacing \emph{hard-min} with a \emph{soft-min} in \eqref{eq:als} leads
to the so-called \emph{Boltzmann Label-Smoothing} (BLS), defined by setting the
term $q_i' \in \Delta_K$ in \eqref{eq:general_q} to:
\begin{eqnarray}
q_i' = \sum_{k=1, k \neq y_i}^K \text{Boltz}^{(k)}_T(x_i,\theta),
\label{eq:boltz}
\end{eqnarray}
where
$\text{Boltz}_T^{(k)}(x_i,\theta) = \frac{\exp(-z^{(j)}(x_i;\theta)/T)}{\sum_{k'=1, k'\neq y_i}^K
  \exp(-z^{(k')}(x_i;\theta)/T)}$
is the Boltzmann distribution with energy levels $z^{(k)}(x_i;\theta)$ at temperature $T \in [0, \infty]$. It interpolates between ALS (corresponding to $T=0$), and SLS (corresponding to $T=\infty$).

\paragraph*{Second-Best Label-Smoothing} SLS, ALS and BLS give positive prediction outputs for every label because we add weight to either every label, or the "worst" wrong label. However, in the problem we consider, it does not matter if we fool the classifier by making it predict the "worst" or the "closest" wrong class. Therefore, a completely different approach consists in concentrating our effort and add all the available mass $\alpha$ only on the "closest" class label.
This leads to \emph{Second Best Label-Smoothing} (SBLS) defined by:
\begin{eqnarray}
q_i' = {\boldsymbol{y^{\text{SB}}_i}} := \argmax_{k=1,\;k\neq {y_i}}^K p^{(k)}(x_i; \theta).
\label{eq:sbls}
\end{eqnarray}
The problem can be rewritten as:

$\displaystyle{\min_{\theta} L_n(\theta) + \alpha \, (-1) \frac{1}{n}\sum_{i=1}^n \max_{k \neq y_i} (z_i^{(k)}) - z_i^{(y_i)}}$.

Note the correspondence with the opposite of the Hinge loss in the second term: this penalty tends to make the margin between the true class prediction and the closest wrong class prediction smaller.

Training with each of these Label-Smoothing methods (SLS, BLS, SBLS) can be implemented via Alg. \ref{alg:als}, using Eqn. \ref{eq:std}, \ref{eq:boltz} or \ref{eq:sbls} respectively, in line 5 of the the Alg. \ref{alg:als} instead of Eqn. \ref{eq:als}. Note also that any other choice of $q'_i$ can lead to different variations of LS, and can be easily implemented the same way.

Here,we finally obtain four different LS methods: ALS, BLS, SBLS and SLS. The effects of ALS in particular and LS as a general method are investigated in Section \ref{sec:toy} and \ref{sec:understanding}, and each of the four methods will be tested as defense methods in Section \ref{exp}. We argue in the next sections that LS as a general method can improve robustness. The choice of the specific method should afterwards be guided by the type of data or problem one is facing, but ALS can be chosen as a default method to improve robustness in the most general cases.

\section{UNDERSTANDING LABEL SMOOTHING}
\subsection{Fading Gaussian Example}
\label{sec:toy}
\vspace{-0.2cm} 
We now explore a simple example illustrating some of the implications of using LS, with regards to standard accuracy and rebustness to adversarial attacks. Let us consider the following problem inspired by \cite{tsipras2018robustness}:
$Y \sim \mathcal{U}(\{-1, 1\})$, and $X_1,...,X_d | Y=y \;  \overset{\text{iid}}{\sim} \mathcal{N}(y \mu_i, \sigma_i^2)$ so that $d$ is the number of features. The possible classifiers studied here are linear classifiers: $f_w: x \mapsto \sign(w^T x)$, with parameters $w \in \mathbb R^d$.

We want to compare the performances (both standard and adversarial and depending only on $w$) of the Bayes classifier and the ALS classifier the above problem.

\begin{figure*}
\vspace{-0.2cm}
\centering
\begin{subfigure}[b]{0.33\textwidth}
  \centering
  \includegraphics[width=1\linewidth]{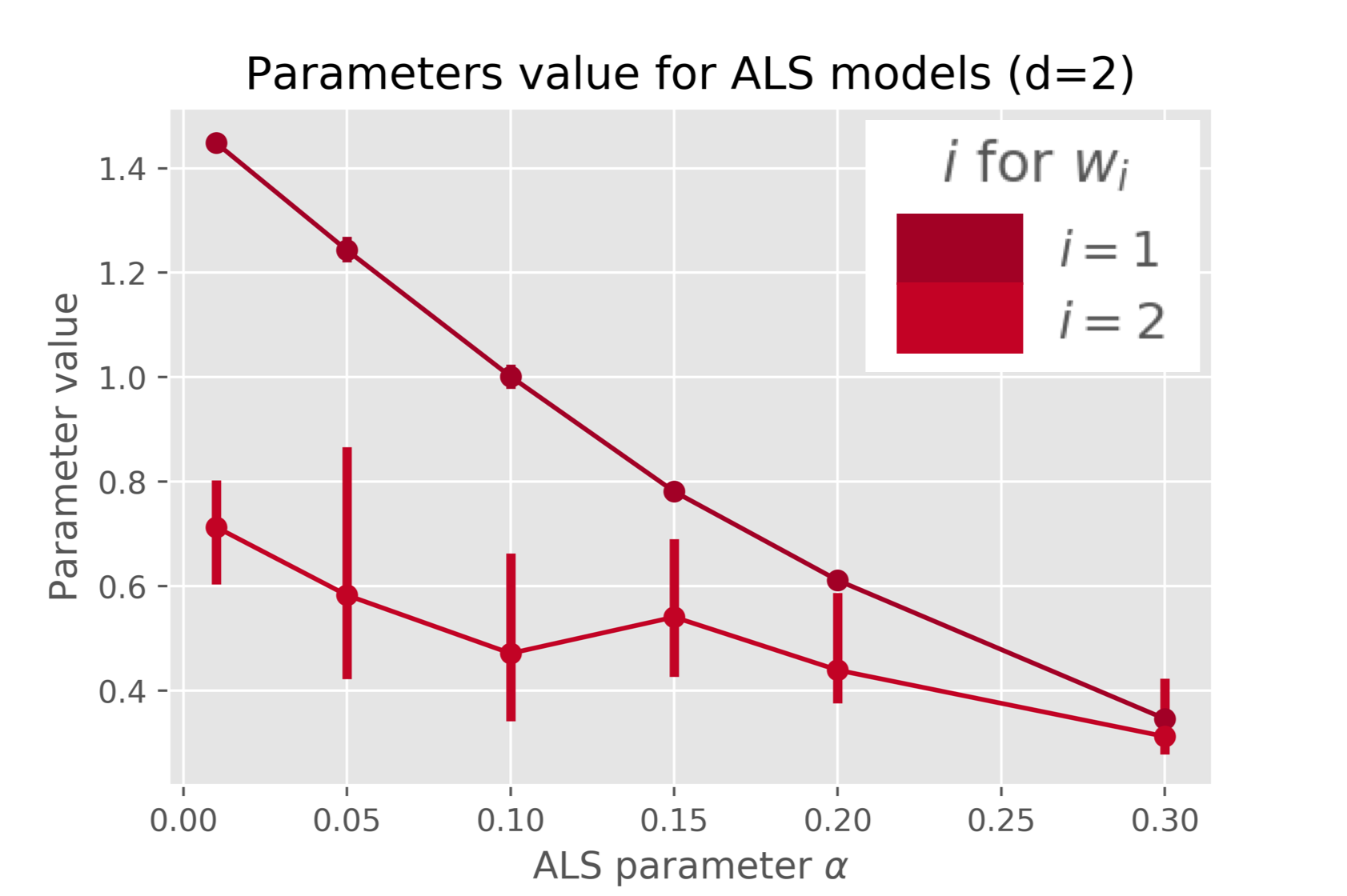}
  \vspace{-0.5cm}
  \caption{d=2}
  \label{fig:fading_gauss_w2}
\end{subfigure}%
\hfill
\begin{subfigure}[b]{0.33\textwidth}
  \centering
  \includegraphics[width=1\linewidth]{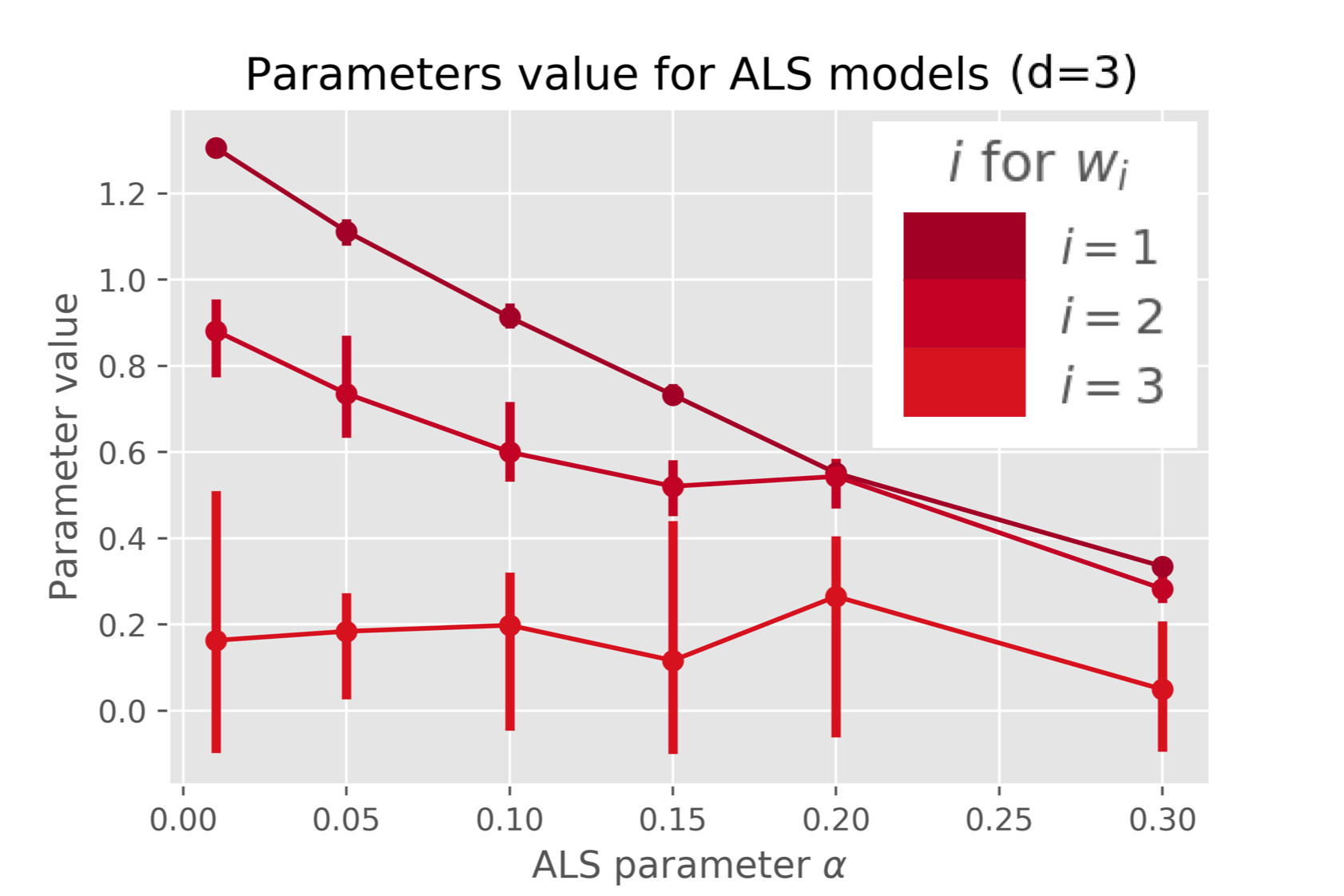}
  \vspace{-0.5cm}
  \caption{d=3}
  \label{fig:fading_gauss_w3}
\end{subfigure}
\hfill
\begin{subfigure}[b]{0.33\textwidth}
  \centering
  \includegraphics[width=1\linewidth]{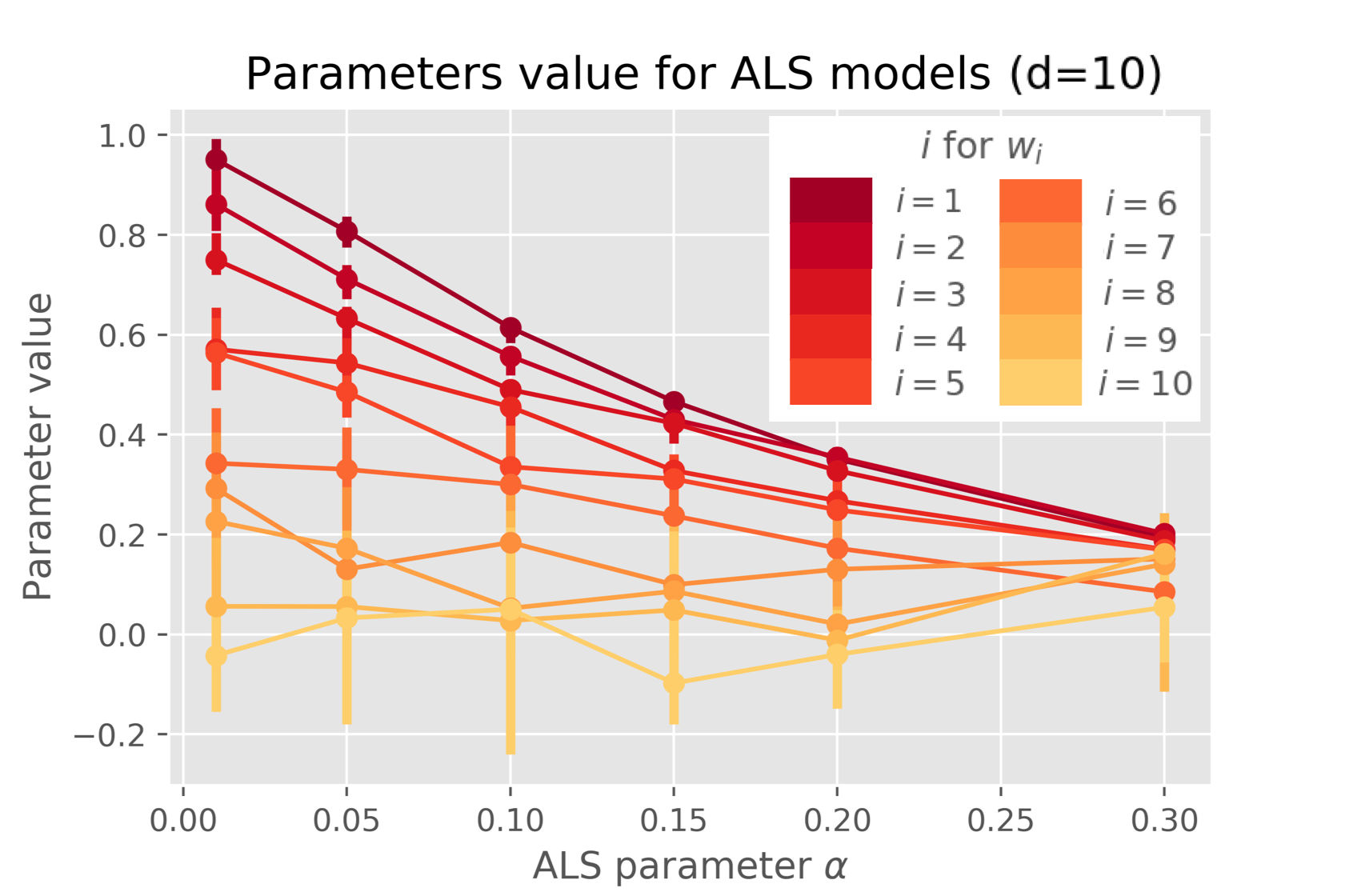}
  \vspace{-0.5cm}
  \caption{d=10}
  \label{fig:fading_gauss_w10}
\end{subfigure}%
\caption{Parameters $w$ in the fading Gaussian experiments}
\label{fig:fading_gauss_w}
\vspace{-0.3cm}
\end{figure*}

\paragraph{Standard accuracy} One computes the adversarial robustness accuracy of the linear model $f_w$ as
\begin{eqnarray}
\begin{split}
&\text{acc}(f_w):=\mathbb{P}( f_w(x) = y ) = \mathbb{P}( yw^T x > 0 )\\
&= \mathbb{P}( \mathcal{N}( w^T\mu ; \sum_j w_j^2 \sigma_j^2) > 0)
=\Psi\left(\frac{w^T\mu}{\sqrt{\sum_j w_j^2 \sigma_j^2}}\right),
\end{split}
\end{eqnarray} where $\Psi$ is the CDF of the the standard Gaussian.

\paragraph{Adversarial accuracy}
For bounded $\ell_\infty$-norm attacks, one computes the adversarial robustness accuracy of the linear model $f_w$ as 
\begin{eqnarray}
\begin{split}
&\text{acc}_\varepsilon(f_w) := \mathbb{P}(f_w(x')=y\;\forall x' \in \mathbb R^d,\;\|x'-x\|_\infty \le \varepsilon)\\
&\quad=\mathbb{P}\left(\min_{\|\Delta x\|_\infty \le \varepsilon}yw^T(x + \Delta x) > 0\right)\\
&\quad= \mathbb{P}( yw^T x - \varepsilon ||w||_1 > 0)
= \Psi \left( \frac{ w^T \mu - \varepsilon ||w||_1}{\sqrt{\sum_{j=1}^d\sigma_j^2w_j^2}} \right),
\end{split}
\end{eqnarray}
where $\varepsilon$ is the strength of the attack.
By the way, the optimal adversarial perturbation is given by 
$x^{adv} = x - \varepsilon y\sign(w)$, where $\sign(w) := (\sign(w_1),\ldots,\sign(w_d))$. When $\varepsilon = 0$, we recover the standard accuracy formula.

We have the following Lemma.
\begin{lemma}
\label{lemma:fading_gaussian}
In the special case where the covariance matrix is diagonal $\Sigma=\diag(\sigma_1^2,\ldots,\sigma_d^2)$, the most robust linear classifier has weights $w \in \mathbb R^d$ given by
\begin{eqnarray}
w_j \propto \sigma_j^{-2}\sign(\mu_j)(|\mu_j| - \varepsilon)_+,
\end{eqnarray}
where $(a)_+ := \max(0, a)$, and the the use of the proportionality symbol "$\propto$" indicates a hidden constant independent of the feature index $j$.

Moreover, the optimal adversarial accuracy amongst all linear classifiers is given by
\begin{eqnarray}
\max_{w \in \mathbb R^d} \text{acc}_\varepsilon(f_w) = \Psi(\sqrt{\Delta(\varepsilon)}),
\end{eqnarray}
where
$\Delta(\varepsilon)=\sum_{j=1}^d\sigma_j^{-2} \left(\left(|\mu_j| - \varepsilon \right)_+ \right)^2$.

\end{lemma}
Thus the best solution for $w$ in terms of adversarial accuracy depends on the feature-wise signal-to-noise ratio (SNR) $\frac{|\mu_j|}{\sigma_j^2}$: if this ratio is high, then the $j$th feature is quite determinant of $y$ because it is both far enough (high mean) from the wrong class and reliable enough (small standard deviation).
\paragraph{Bayes classifier} One computes the posterior probability of the positive class label given a sample $x$ as $\mathbb{P}(Y=1 | x) = (1+\text{e}^{ -2 \sum \frac{\mu_i}{\sigma_j^2} x_j})^{-1}$, and so $\mathbb{P}(Y=1 | x) > 1/2 \Leftrightarrow \sum_{j=1}^d \frac{\mu_j}{\sigma_j^2} x_j > 0$. Thus the Bayes-optimal classifier is linear with parameters given by $w_j = \mu_j / \sigma_j^2,\;\forall j$. Thus in particular, if $\mu_j / \sigma_j^2=\text{ constant} \;\forall j$ (as in \cite{tsipras2018robustness}), then the Bayes optimal classifier corresponds to the linear model with parameters $w = (1,\ldots,1)$.

\paragraph{ALS classifier} The ALS problem reduced to finding $w$ minimizing the expected value of the smooth cross-entropy loss defined in \ref{contrib}. Let us note $p_1$ and $p_{-1}$ the predicted probabilities of each class, so $p_1 = 1/(1+ \text{e}^{-w^Tx})$ and $p_{-1} = 1-p_1$. Depending on the values of $x$ and $y$, the loss can take different forms:
\begin{itemize}
\vspace{-0.3cm}
    \item $l_1(x;w) = \alpha \ln{(p_{-1})} + (1-\alpha)\ln{(p_{1})}$ if $y=1$ and $w^Tx > 0$ \vspace{-0.1cm}
    \item $l_2(x;w) = \ln{(p_1)}$ if $y=1$ and $w^Tx < 0$ \vspace{-0.1cm}
    \item $l_3(x;w) = \ln{(p_{-1})}$ if $y=-1$ and $w^Tx > 0$ \vspace{-0.1cm}
    \item $l_4(x;w) = (1-\alpha) \ln{(p_{-1})} + \alpha\ln{(p_{1})}$ if $y=-1$ and $w^Tx < 0$, \text{thus}
\end{itemize}
\vspace{-0.1cm}
\begin{equation*}
\label{eq:ALS_fading_gaussian}
\begin{split}
&\mathbb{E}_{X,Y}(loss_w(X,Y)) = \frac{1}{2} \left[ \int_{w^T x>0} l_1(x;w) f_{Y=1}(x) \, \text{d}x \right. + \\
& \int_{w^T x<0} l_2(x;w) f_{Y=1}(x) \, \text{d}x + \int_{w^T x>0} l_3(x;w) f_{Y=-1}(x) \, \text{d}x \\
& +\left. \int_{w^T x<0} l_4(x;w) f_{Y=-1}(x) \, \text{d}x \right]
\end{split}
\end{equation*} ~\\[-0.2cm]
where $f_{Y=1}$ (resp. $f_{Y=-1}$) is the density of $X|Y=1$ (resp. $X|Y=-1$)

The derivative of equation \ref{eq:ALS_fading_gaussian} is difficult to express in close form. The values for $w$ obtained when running the experiments (see below) can be derived from Figure \ref{fig:fading_gauss_w}. We easily check that the features are "correctly ordered" in terms of parameter $w$ : the more important the feature is, higher the $w$ is, which leads to better adversarial accuracy.

\begin{figure*}[!t]
\centering
\begin{subfigure}[b]{0.49\textwidth}
  \centering
  \includegraphics[width=0.98\linewidth]{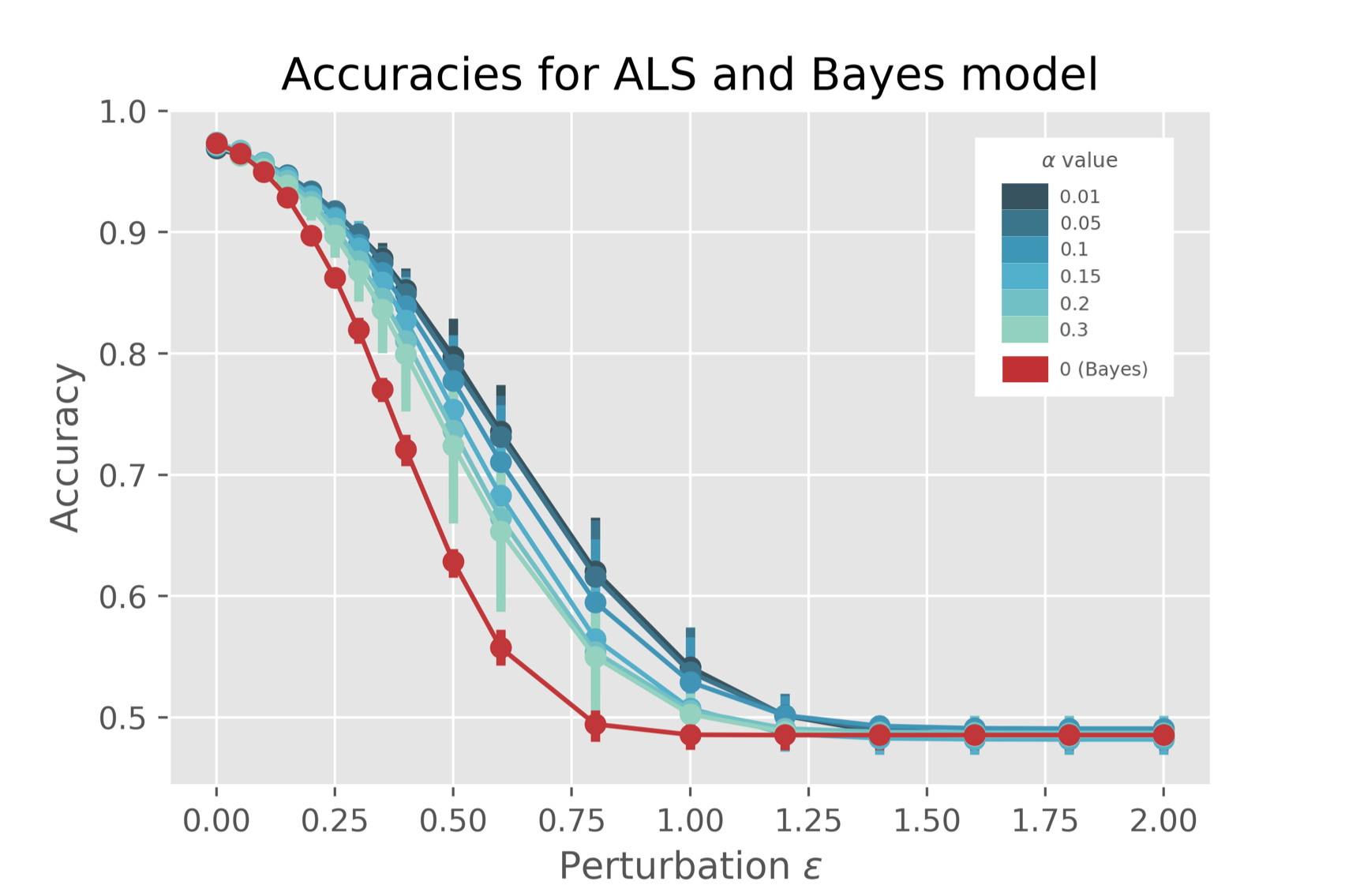}
  \caption{\textbf{Results:} Fading gaussian experiment for d=10 features. ALS accuracy is better especially for small $\alpha$}
  \label{fig:fading_gauss_10}
\end{subfigure}%
\hfill
\begin{subfigure}[b]{0.49\textwidth}
  \centering
  \includegraphics[width=0.98\linewidth]{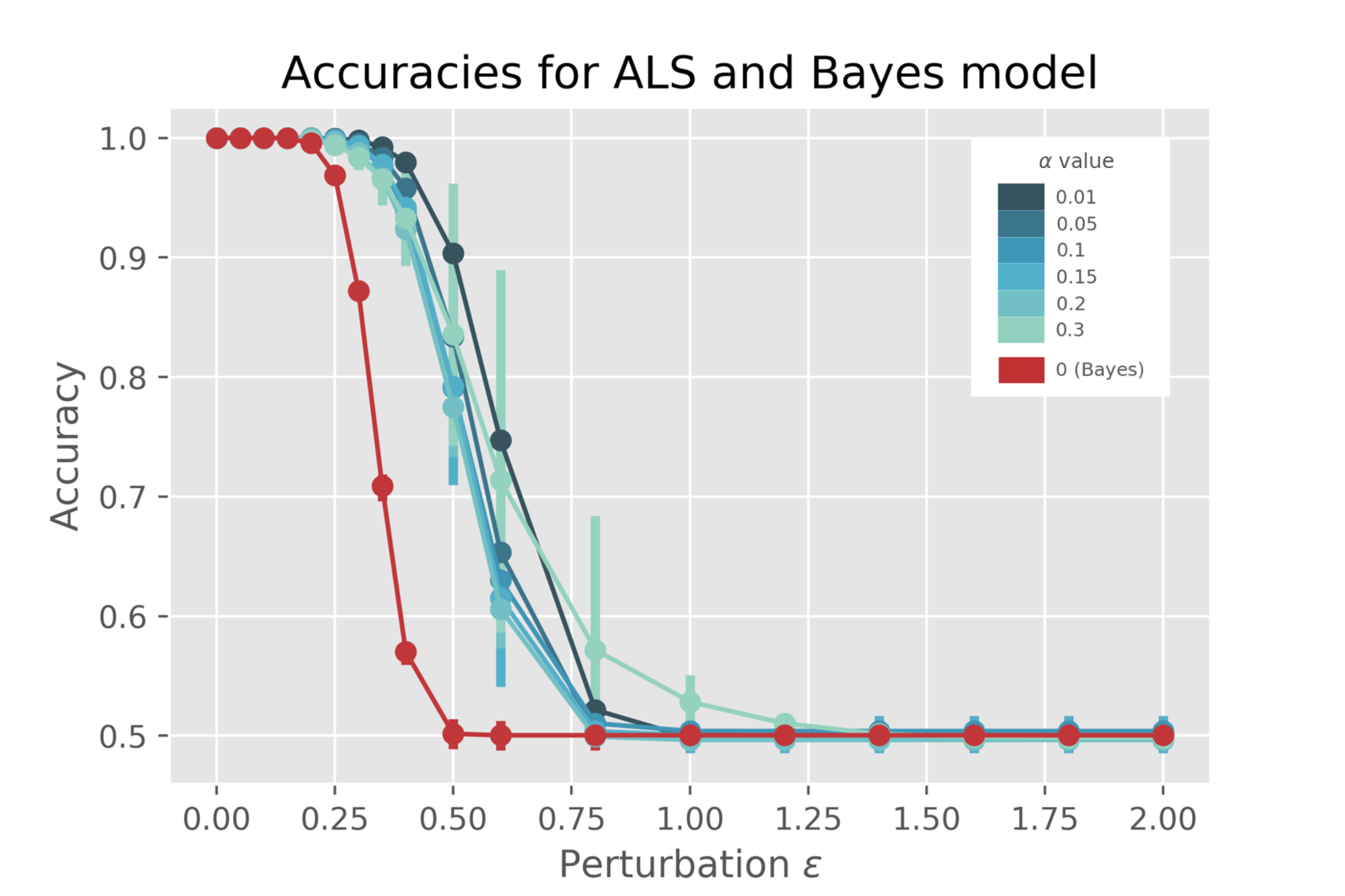}
  \caption{\textbf{Results:} Fading gaussian exp. for d=100 features. ALS accuracy is also better especially for small $\alpha$}
  \label{fig:fading_gauss_100}
\end{subfigure}
\caption{Fading gaussian experiment}
\label{fig:fading_gauss_results}
\vspace{-0.3cm}
\end{figure*}

\paragraph{Empirical illustration} We ran some experiments choosing $\forall \, i \in [\![1, d]\!], \sigma_i = 1 - (i-1)/d$ and so $\mu_i = \sigma_i^2$.

Figure \ref{fig:fading_gauss_results} shows the standard and adversarial accuracies for the Bayes classifier as well as ALS classifiers for different values of $\alpha$, as a function of the perturbation $\varepsilon$. We see that even for $d$ as small as $d=10$, the ALS classifiers do better (especially for small values of $\alpha$) than the Bayes classifier under the adversarial regime, which is consistent with the $w$ values obtained in Fig. \ref{fig:fading_gauss_w} and Lemma \ref{lemma:fading_gaussian}. Importantly, the standard accuracy ($\varepsilon = 0$) is not so different between Bayes and ALS classifiers, suggesting that ALS enables better adversarial generalization without damaging the "natural" one.

Figure \ref{plot:kde} shows the kernel density plot (KDE) of the dataset when $d=2$. The black line is the main direction of the density, while the two red lines are different decision boundaries for two different choices of $w$. The plain red line is the decision boundary corresponding to the Bayes case, so $w = (1,1)$, while the dashed red line corresponds here to the case $w = (4,1)$ and is orthogonal to the black line.

\begin{figure}[H]
  \begin{center}
   \vspace{-0.3cm}
    \includegraphics[width=0.42\textwidth]{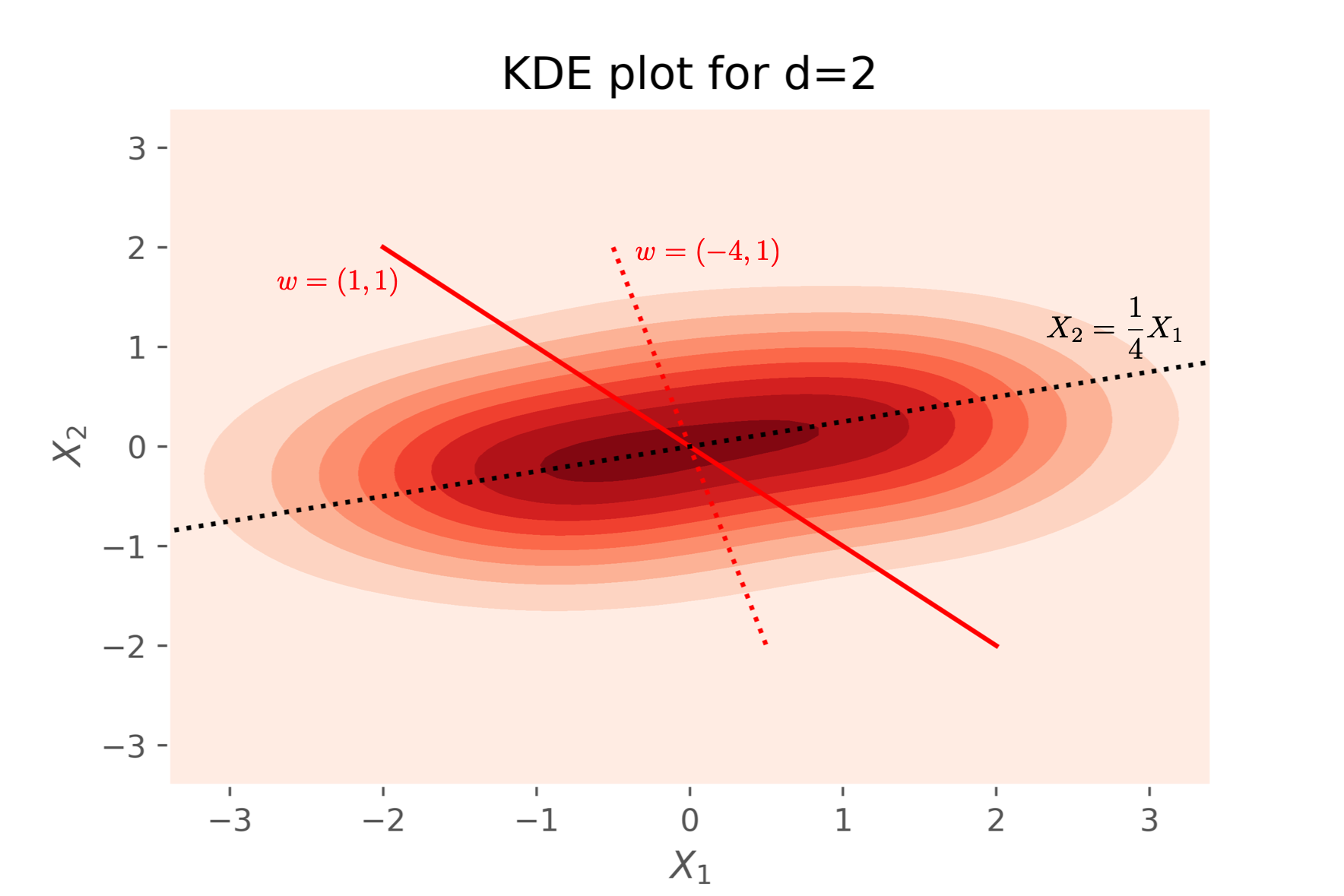}
  \end{center}
  \vspace{-0.5cm}
  \caption{KDE plot: adversarial acc. dashed decision boundary is higher than plain's one.}
  \label{plot:kde}
  \vspace{-0.4cm}
\end{figure}

It is interesting to realize that potential adversarial examples are points that lay around the decision boundary (of course, at a distance depending on the choice of the perturbation $\varepsilon$). Thus, it is more efficient, under the adversarial regime, to have a decision boundary that crosses regions of low density, or spends the shortest possible time in regions of high density. This is exactly what does the dashed red decision boundary, compared to the plain red/Bayes decision boundary that does not take into account this phenomenon. This is equivalent to say that in order to gain better robustness, a classifier must use the information provided by the variability of the features, a fact which was already apparent in the formula for $w_j$ derived in the adversarial accuracy paragraph.

\subsection{A Closer Look At Label-Smoothing}
\label{sec:understanding}

\begin{figure*}[t]
\centering
\begin{subfigure}[b]{0.49\textwidth}
  \centering
  \includegraphics[width=0.95\linewidth]{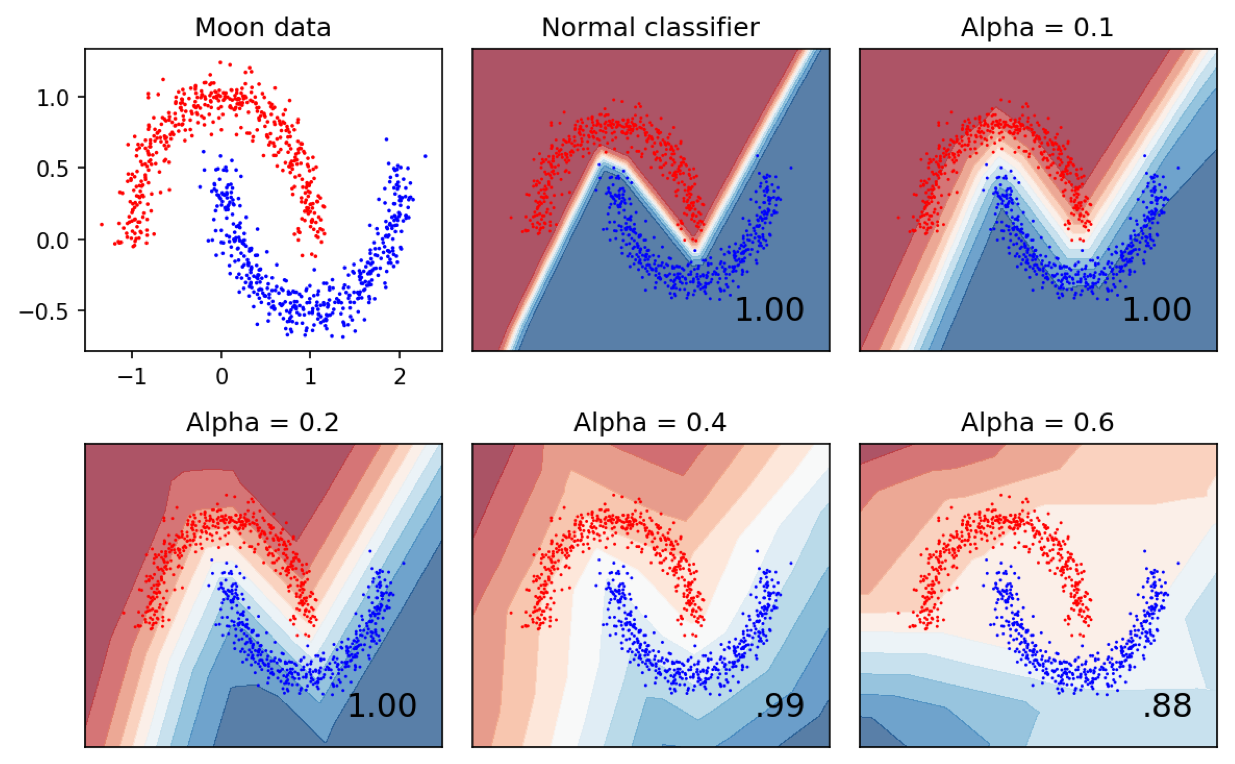}
  \caption{\textbf{Regularization effect:} logit squeezing using ALS (different $\alpha$) and a MLP classifier. Darker is more confidence.}
  \label{fig:regularization}
\end{subfigure}%
\hfill
\begin{subfigure}[b]{0.49\textwidth}
  \centering
  \includegraphics[width=0.95\linewidth]{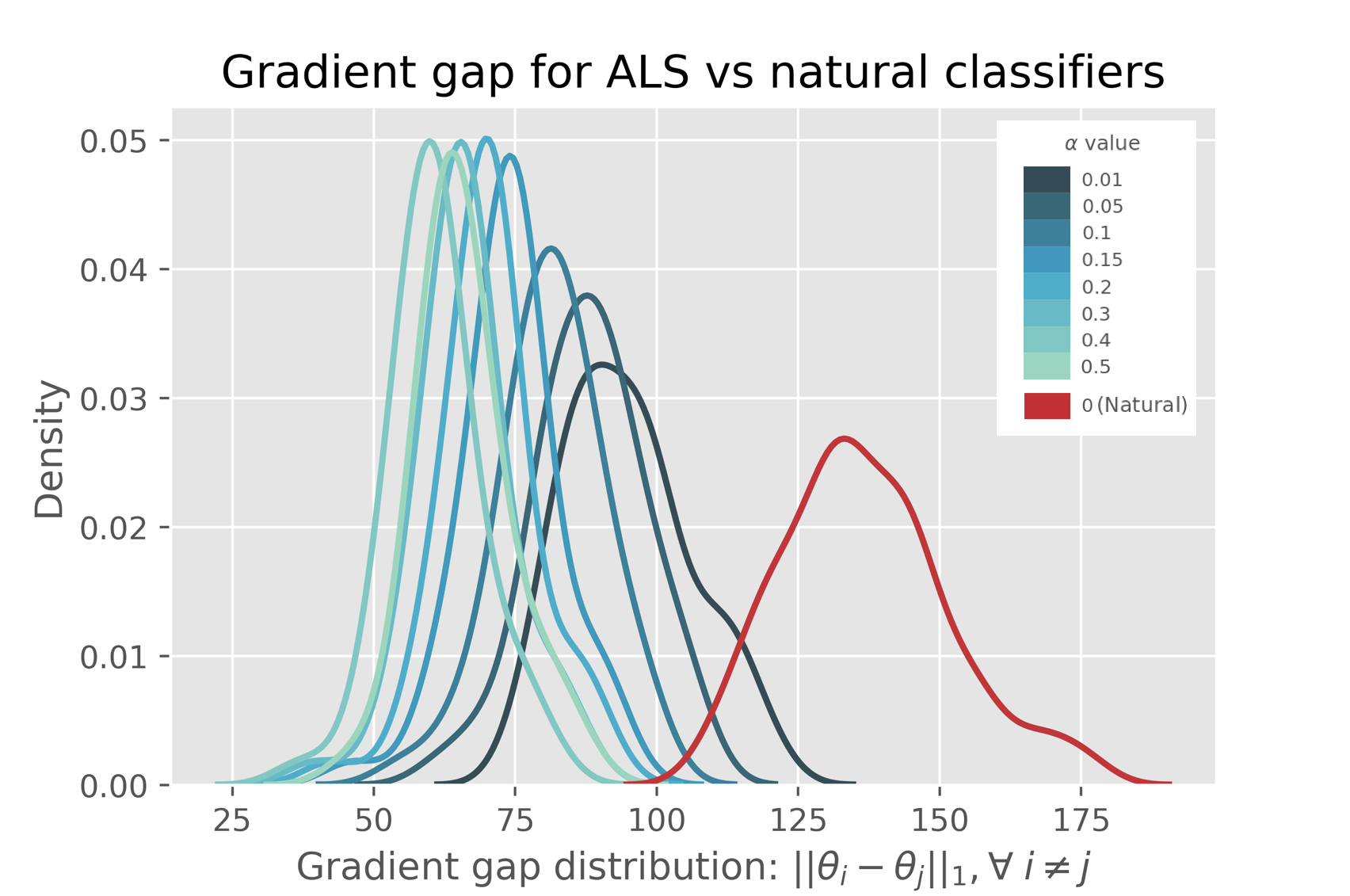}
  \caption{\textbf{Gradient gap reduction} for a one-layer linear classifier on MNIST.}
  \label{fig:gradient_gap}
\end{subfigure}
\caption{Effects of LS}
\label{fig:triangular}
\vspace{-0.3cm}
\end{figure*}

\subsubsection{Logit-squeezing and gradient-based methods}
\label{subsec:logit}
Applying label-smoothing (LS) generates a logit-squeezing effect (see Theorem \ref{thm:main_thm}) which tends to prevent the model from being over-confident in its predictions. This effect was investigated in \cite{pereyra2017regularizing} and is illustrated in Fig. \ref{fig:regularization}, where we plot the prediction values for different MLP classifiers trained on the moon dataset.

In addition to this impact on the logits and predictions, LS also have an effect on the logits' gradients (with respect to $x$, see Fig. \ref{fig:gradient_gap} where we plot these gradients for one-layer linear classifiers -ALS and natural- trained on MNIST) which can help explain why ALS trained models are more robust to adversarial attacks. As described in \cite{shafahi2018label}, using a linear approximation, an attack is successful if

$p_y(x, \theta) + \delta^T \nabla_x p_y(x, \theta) \leq p_k(x, \theta) + \delta^T \nabla_x p_k(x, \theta)$

for any $k \neq y$, where $\delta$ is the attack perturbation. With an FGSM-like attack of strength $\epsilon$, it thus works if

$ \epsilon \geq \frac{1}{|| \nabla_x z_y(x, \theta) - \nabla_x z_j(x, \theta) ||_1}$.

By reducing this gradient gap, LS provides more robust models at least against gradient-based attack methods.

\subsubsection{Why does LS help adversarial robustness?}
\label{subsec:intuition}
Pointwise, the SmoothCE loss induces different costs compared to the traditional CE loss. As discussed in Sec. \ref{subsec:logit}, over-confidently classified points are more penalized. Likewise, very badly classified points are also more penalized. The model is thus forced to put the decision boundary in a region with few data points (see Section \ref{sec:toy}). If not, either the penalty term $R_n(\theta)$ or the general term $L_n(\theta)$, defined in Sec. \ref{subsec:als}, will be too high.
The underlying geometry of the dataset is thus better addressed compared to a traditional training: boundaries are closer to "the middle", i.e the margin between two classes is bigger (similar to how SVM operates), leading to increased robustness. Traditional CE loss, however, induces a direct power relationship: the boundary between two classes is pushed close to the smallest one.

\section{EXPERIMENTS}
\label{exp}


\begin{figure*}[!t]
\vspace{-0.2cm}
\centering
  \begin{subfigure}{0.47\linewidth}
    \includegraphics[width=1\linewidth]{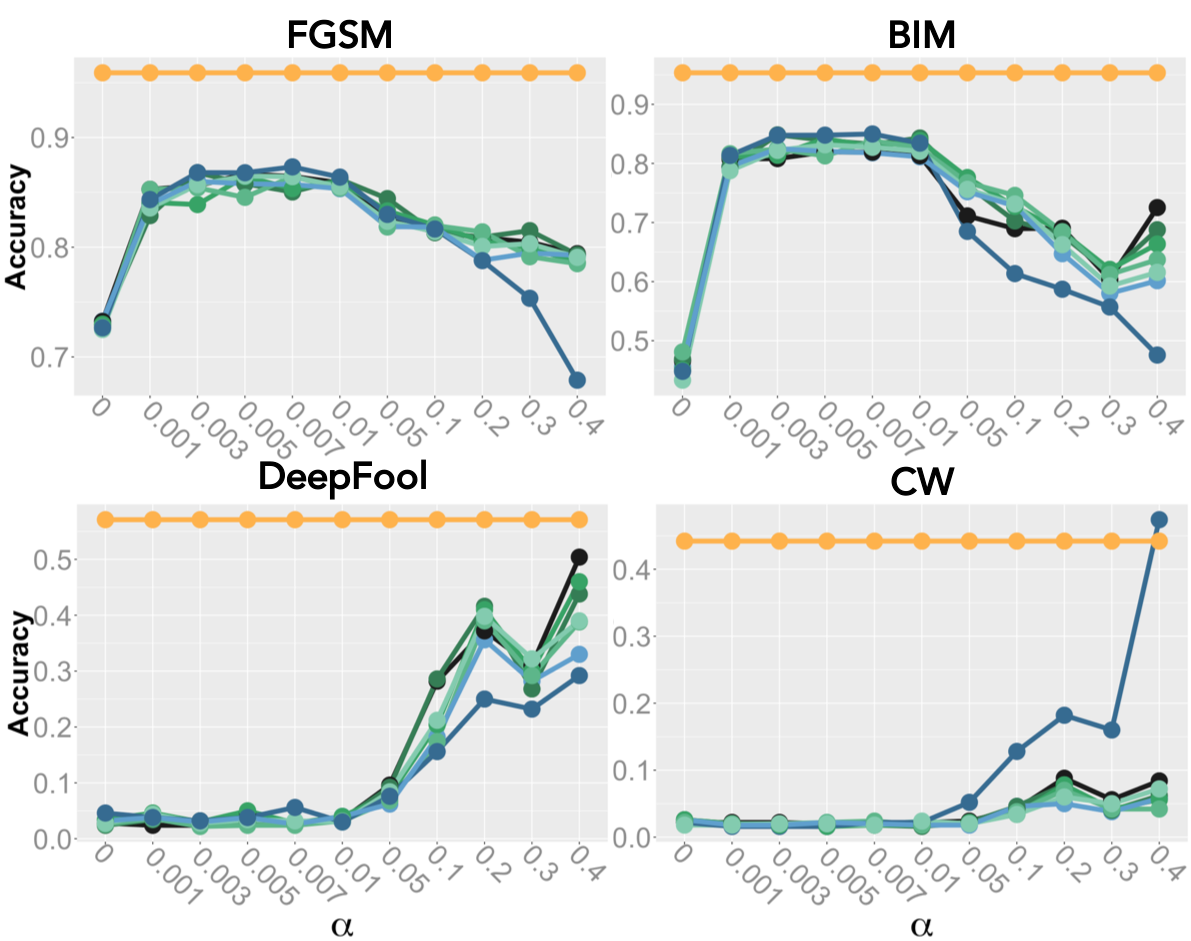}
    \caption{\textbf{MNIST LeNet.} $\epsilon = 0.15$ for FGSM and BIM}
    \label{fig:mnist}
  \end{subfigure} \hfill
  \begin{subfigure}{0.47\linewidth}
    \includegraphics[width=1\linewidth]{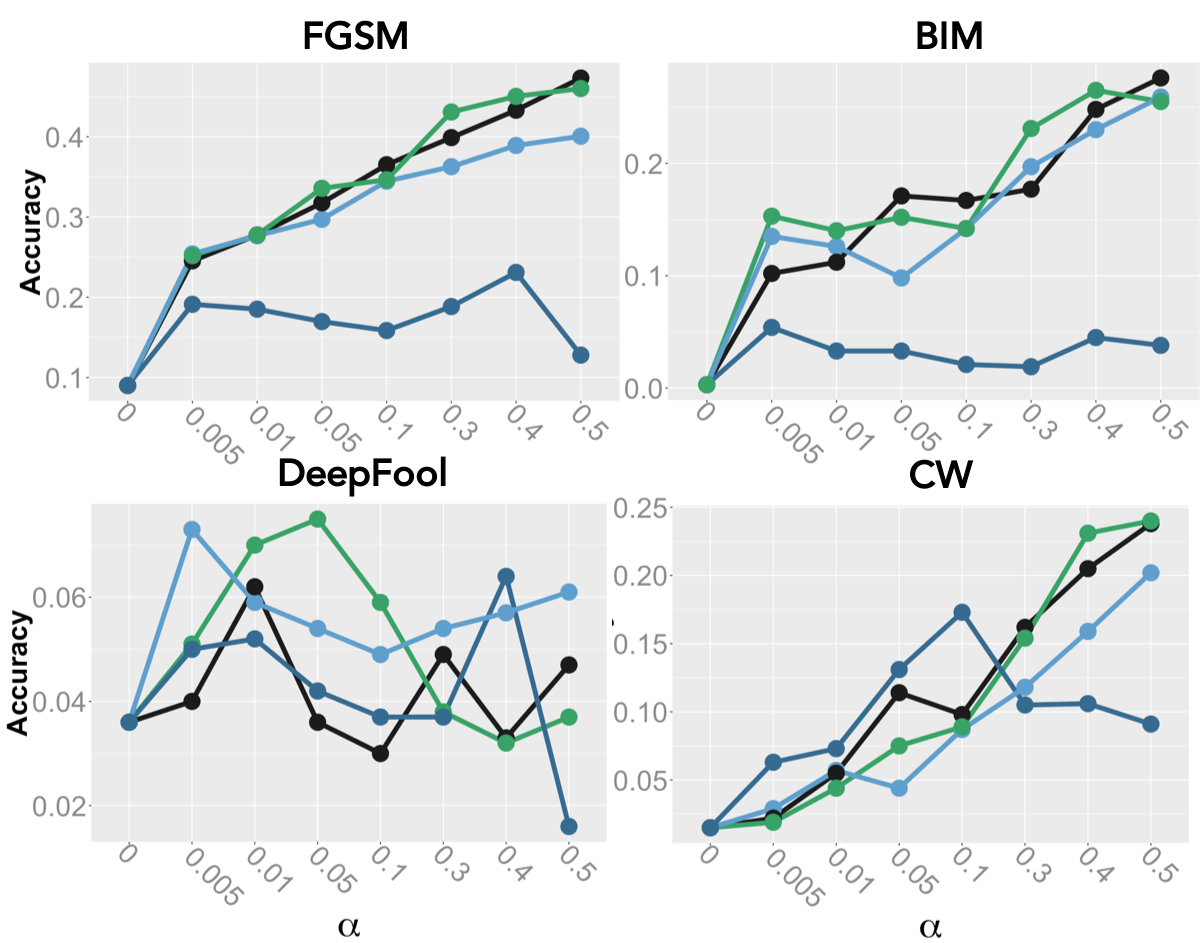}
    \caption{\textbf{CIFAR10 ResNet}. $\epsilon = 0.05$ for FGSM and BIM}
    \label{fig:cifar}
  \end{subfigure}
  
\bigskip
\vspace{-0.3cm}

  \begin{subfigure}{0.47\linewidth}
    \includegraphics[width=1\linewidth]{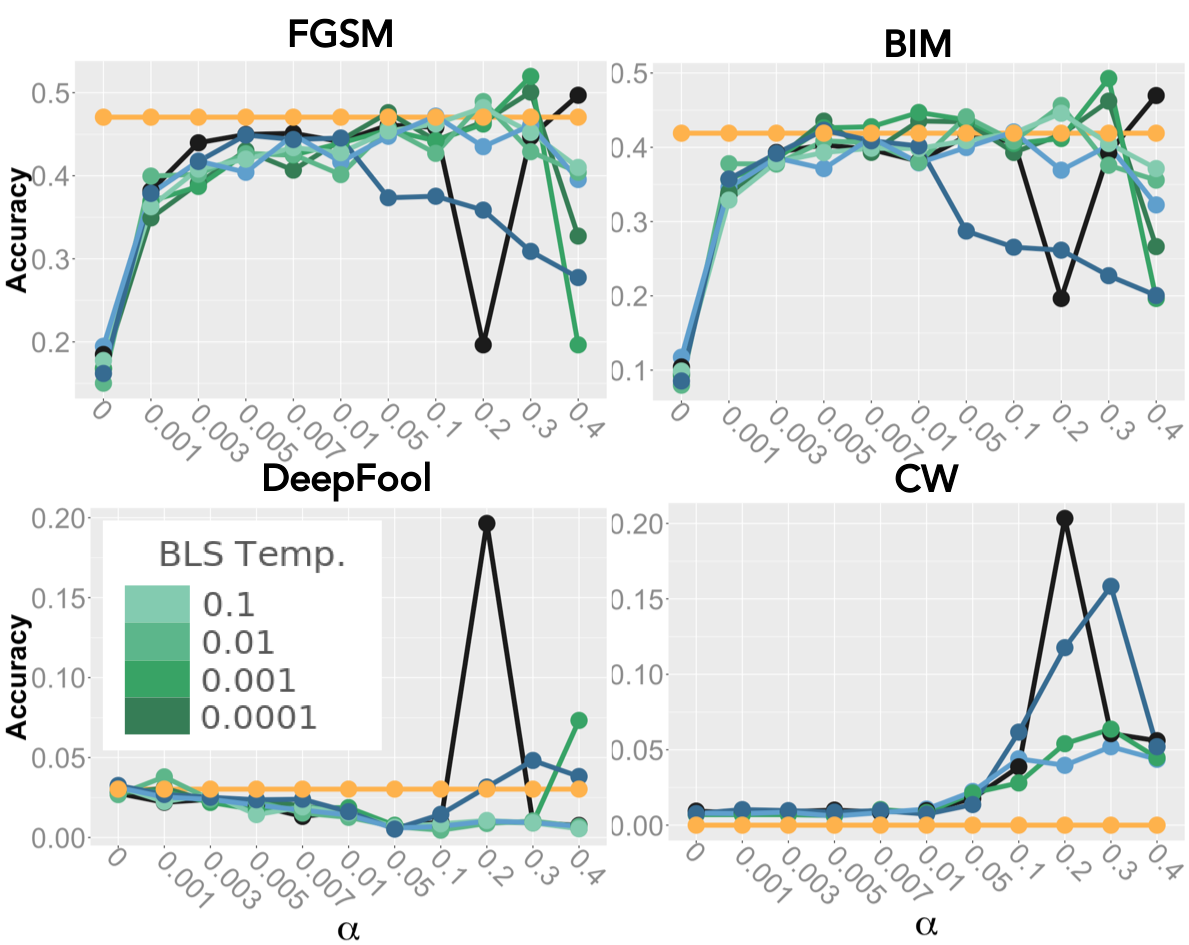}
    \caption{\textbf{SVHN LeNet}. $\epsilon = 0.05$ for FGSM and BIM}
    \label{fig:svhn}
  \end{subfigure} \hfill
  \begin{subfigure}{0.47\linewidth}
    \includegraphics[width=1\linewidth]{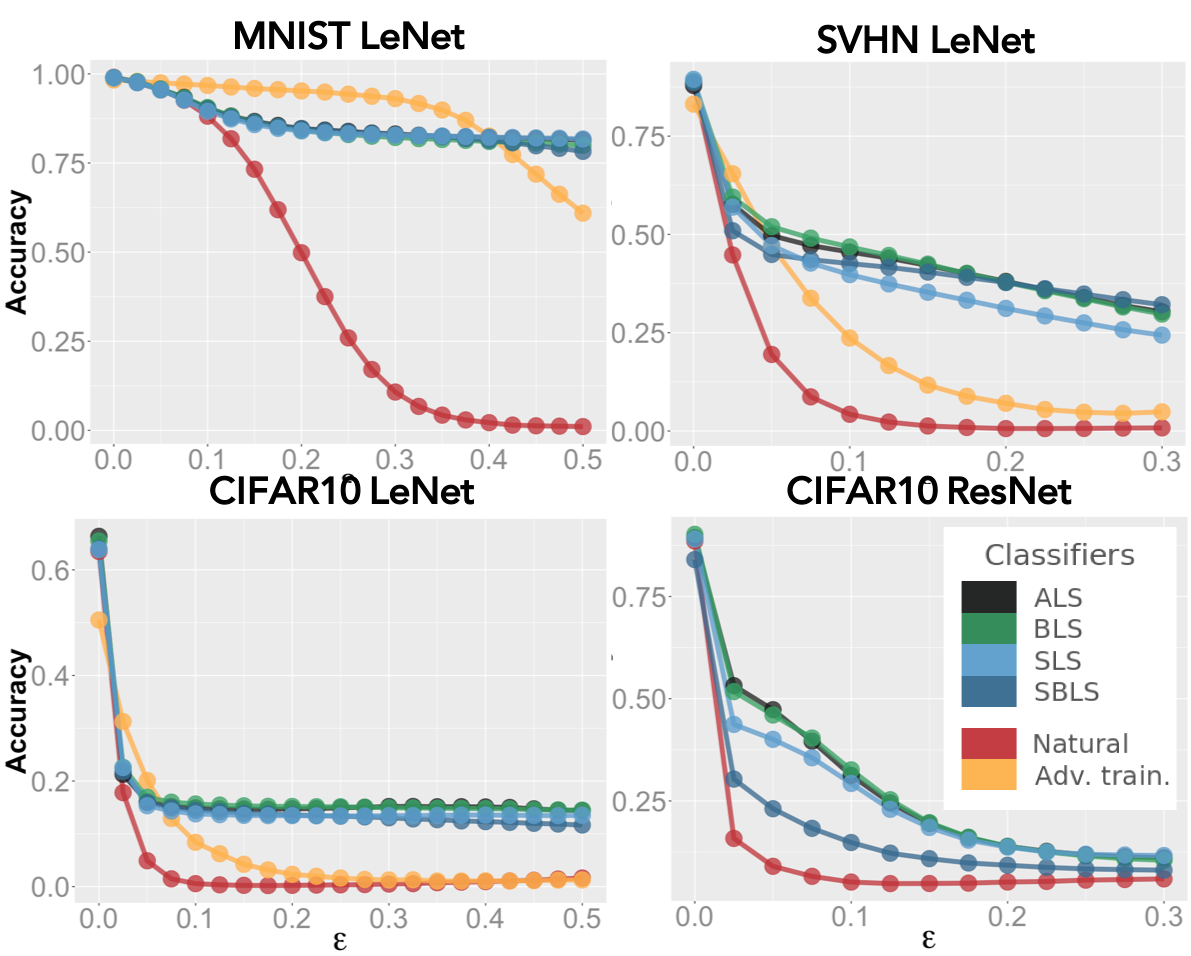}
    \caption{\textbf{FGSM:} all set-ups}
    \label{fig:fgsm}
   \end{subfigure}
  \caption{\textbf{Experiments:} Figs \ref{fig:mnist}, \ref{fig:cifar}, \ref{fig:svhn} show the evolution of adversarial accuracy as a function of $\alpha$ for different models against all attacks. Fig \ref{fig:fgsm} shows it as a function of $\epsilon$ for all models against FGSM.}  
  \vspace{-0.2cm}
\end{figure*}

\paragraph{Set-up}
We run the four different attacks\footnote{In the CIFAR10 ResNet and SVHN LeNet set-ups, the number of iteration for C\&W attack is sub-optimal. More iterations would have broken the models, however, C\&W is hardly scalable because it requires a long time to train, especially for sophisticated models like ResNet.} on different set ups (datasets MNIST, CIFAR10, SVHN where the pixel range is 1, and models MLP, LeNet, ResNet18). For comparison purposes, we also run the attacks on reference models: "natural" classifier, i.e. the same models used in the experiments but without any LS or regularization (see red lines in Fig. \ref{fig:fgsm} and $\alpha = 0$ values for Figs. \ref{fig:mnist}, \ref{fig:cifar}, \ref{fig:svhn}); and models trained using adversarial training against PGD as defined in \cite{madry2017towards} with 3 iterations (see yellow lines in Fig. \ref{fig:results}). The parameters used here for the PGD-training are: for MNIST (Linear and LeNet), $\epsilon_{PGD} = 0.25$ and $\alpha_{PGD} = 0.1$ (same notations as in \cite{madry2017towards}), and for every other set-up, $\epsilon_{PGD} = 0.05$ and $\alpha_{PGD} = 0.02$. PGD-training is a SOTA defense, and we chose to use 3 iterations here once again for comparative purposes, but this time with respect to training time: our method is as fast as a natural training, but PGD takes way longer to train (in this case with 3 iterations, the training time was at least 3 time longer on every set-up than for ALS training, sometimes even 15 times longer). PGD-training is not run on the CIFAR10 ResNet18 set-up because of this long training time.

\paragraph{Results}
Some results are shown in Fig. \ref{fig:results}, and see Tables \ref{tab:fgsm} to \ref{tab:accuarcy} (Appendix \ref{app:tables}) for more extensive results.

One can see that Label-Smoothing performs always better in terms of robustness than the natural classifiers, which indicates that LS indeed provides some kind of robustness. Moreover, our LS methods are competitive with PGD-training. On MNIST LeNet (see fig. \ref{fig:mnist}), PGD training yields better results against FGSM, BIM and DeepFool, but in these three cases, our LS methods are not far from PGD-training and SBLS is better than PGD-training against C\&W attack. On the other set-ups, LS is better against every attack except FGSM or BIM only when the strenght of the attack $\varepsilon$ is very small (but note that for SVHN and CIFAR10, the standard accuracy with PGD-training is weakened). On the whole, ALS and BLS give better results than SLS and SBLS. They thus should be preferred as default methods when implementing LS. However, overall, there is no major differences in the results obtained with the different LS methods, and the temperature hyperparameter for BLS method does not seem to have a great impact on the results. ($T=0.001$ for example is a good default value).

Altogether, we see that LS is a good candidate for improving the adversarial robustness of NNs, eventhough it will not deliver perfectly robust model. Still, the absence of computation cost makes it a very useful robustification method for lots of real life applications where NNs are required to be trained quickly.

\section{CONCLUSION}
\label{conclu}
We have proposed a general framework for Label-Smoothing (LS) as well as a new variety of LS methods (Section \ref{contrib}) as a way to alleviate the vulnerability of Deep learning image classification algorithms. We developed a theoretical understanding of LS (Theorem \ref{thm:main_thm} and Section \ref{sec:understanding}) and our results have been demonstrated empirically via experiments on real datasets (CIFAR10, MNIST, SVHN), neural-network models (MLP, LeNet, ResNet), and
SOTA
attack models (FGSM, BIM, DeepFool, C\&W).

LS improves the adversarial accuracy of neural networks, and can also boost standard accuracy, suggesting a connection between adversarial robustness and generalization. Even though our results (see Section \ref{sec:understanding}) provide evidence that LS classifiers are more robust because they take the dataset geometry into better consideration, better understanding of the adversarial phenomenon and the representations learned by NNs would be desirable.

Moreover, compared to other defense methods (e.g adversarial training), the ease of implementation of LS is very appealing: it is simple, fast, with one interpretable hyperparameter ($\alpha \in [0, 1]$). Being costless is one of the major benefits of implementing LS.
Experimental results (section \ref{exp}) could be completed with various NNs and datasets, which is also left for future works.

\clearpage
\bibliography{bibliography}

\begin{thebibliography}{10}

\bibitem{akhtar2018threat}
Naveed Akhtar and Ajmal Mian.
\newblock Threat of adversarial attacks on deep learning in computer vision: A
  survey.
\newblock {\em IEEE Access}, 6:14410--14430, 2018.

\bibitem{carlini2017towards}
Nicholas Carlini and David Wagner.
\newblock Towards evaluating the robustness of neural networks.
\newblock In {\em 2017 IEEE Symposium on Security and Privacy (SP)}, pages
  39--57. IEEE, 2017.

\bibitem{fawzi2016robustness}
Alhussein Fawzi, Seyed-Mohsen Moosavi-Dezfooli, and Pascal Frossard.
\newblock Robustness of classifiers: from adversarial to random noise.
\newblock In {\em Advances in Neural Information Processing Systems}, pages
  1632--1640, 2016.

\bibitem{goodfellow2014explaining}
Ian~J Goodfellow, Jonathon Shlens, and Christian Szegedy.
\newblock Explaining and harnessing adversarial examples.
\newblock {\em arXiv preprint arXiv:1412.6572}, 2014.

\bibitem{krizhevsky2012imagenet}
Alex Krizhevsky, Ilya Sutskever, and Geoffrey~E Hinton.
\newblock Imagenet classification with deep convolutional neural networks.
\newblock In {\em Advances in neural information processing systems}, pages
  1097--1105, 2012.

\bibitem{kurakin2016adversarial}
Alexey Kurakin, Ian Goodfellow, and Samy Bengio.
\newblock Adversarial machine learning at scale.
\newblock {\em arXiv preprint arXiv:1611.01236}, 2016.

\bibitem{madry2017towards}
Aleksander Madry, Aleksandar Makelov, Ludwig Schmidt, Dimitris Tsipras, and
  Adrian Vladu.
\newblock Towards deep learning models resistant to adversarial attacks.
\newblock {\em arXiv preprint arXiv:1706.06083}, 2017.

\bibitem{moosavi2016deepfool}
Seyed-Mohsen Moosavi-Dezfooli, Alhussein Fawzi, and Pascal Frossard.
\newblock Deepfool: a simple and accurate method to fool deep neural networks.
\newblock In {\em Proceedings of the IEEE conference on computer vision and
  pattern recognition}, pages 2574--2582, 2016.

\bibitem{papernot2016effectiveness}
Nicolas Papernot and Patrick McDaniel.
\newblock On the effectiveness of defensive distillation.
\newblock {\em arXiv preprint arXiv:1607.05113}, 2016.

\bibitem{papernot2016distillation}
Nicolas Papernot, Patrick McDaniel, Xi~Wu, Somesh Jha, and Ananthram Swami.
\newblock Distillation as a defense to adversarial perturbations against deep
  neural networks.
\newblock In {\em 2016 IEEE Symposium on Security and Privacy (SP)}, pages
  582--597. IEEE, 2016.

\bibitem{pereyra2017regularizing}
Gabriel Pereyra, George Tucker, Jan Chorowski, Lukasz Kaiser, and Geoffrey
  Hinton.
\newblock Regularizing neural networks by penalizing confident output
  distributions.
\newblock {\em arXiv preprint arXiv:1701.06548}, 2017.

\bibitem{shafahi2018label}
Ali Shafahi, Amin Ghiasi, Furong Huang, and Tom Goldstein.
\newblock Label smoothing and logit squeezing: A replacement for adversarial
  training?
\newblock 2018.

\bibitem{sitawarin2018darts}
Chawin Sitawarin, Arjun~Nitin Bhagoji, Arsalan Mosenia, Mung Chiang, and
  Prateek Mittal.
\newblock Darts: Deceiving autonomous cars with toxic signs.
\newblock {\em arXiv preprint arXiv:1802.06430}, 2018.

\bibitem{labelsmoothing}
C.~{Szegedy}, V.~{Vanhoucke}, S.~{Ioffe}, J.~{Shlens}, and Z.~{Wojna}.
\newblock Rethinking the inception architecture for computer vision.
\newblock In {\em 2016 IEEE Conference on Computer Vision and Pattern
  Recognition (CVPR)}, pages 2818--2826, June 2016.

\bibitem{szegedy2013intriguing}
Christian Szegedy, Wojciech Zaremba, Ilya Sutskever, Joan Bruna, Dumitru Erhan,
  Ian Goodfellow, and Rob Fergus.
\newblock Intriguing properties of neural networks.
\newblock {\em arXiv preprint arXiv:1312.6199}, 2013.

\bibitem{tanay2016boundary}
Thomas Tanay and Lewis Griffin.
\newblock A boundary tilting persepective on the phenomenon of adversarial
  examples.
\newblock {\em arXiv preprint arXiv:1608.07690}, 2016.

\bibitem{tramer2017space}
Florian Tram{\`e}r, Nicolas Papernot, Ian Goodfellow, Dan Boneh, and Patrick
  McDaniel.
\newblock The space of transferable adversarial examples.
\newblock {\em arXiv preprint arXiv:1704.03453}, 2017.

\bibitem{tsipras2018robustness}
Dimitris Tsipras, Shibani Santurkar, Logan Engstrom, Alexander Turner, and
  Aleksander Madry.
\newblock Robustness may be at odds with accuracy.
\newblock {\em stat}, 1050:11, 2018.

\bibitem{labelsmoothingbis}
David Warde-Farley.
\newblock Adversarial perturbations of deep neural networks.
\newblock 2016.

\bibitem{zhang2018adversarial}
Jiliang Zhang and Xiaoxiong Jiang.
\newblock Adversarial examples: Opportunities and challenges.
\newblock {\em arXiv preprint arXiv:1809.04790}, 2018.

\bibitem{zheng2018improvement}
Qinghe Zheng, Mingqiang Yang, Jiajie Yang, Qingrui Zhang, and Xinxin Zhang.
\newblock Improvement of generalization ability of deep cnn via implicit
  regularization in two-stage training process.
\newblock {\em IEEE Access}, 6:15844--15869, 2018.

\end{thebibliography}
\bibliographystyle{plain}


\clearpage
\setcounter{section}{0}
\section{Appendix}
\label{appendix}

\subsection{LS optimization program}
\label{app:ls_program}

\begin{proof}[Proof of Theorem \ref{thm:main_thm}]
Recall that $\byi \in \Delta_K$ is the one-hot encoding of the example $x_i$ with label $y_i \in [\![1,K]\!]$. By direct computation, one has
\begin{align*}
    \frac{1}{n}\sum_{i=1}^n &\operatorname{SmoothCE}(x_i, q_i;\theta) = \frac{1}{n}\sum_{i=1}^n q_i^T \ln( p(x_i;\theta) ) \\
    &= -\frac{1}{n}\sum_{i=1}^n ( (1-\alpha)\boldsymbol{y_i} + \alpha q'_i )^T \ln( p(x_i;\theta) ) \\
    & = -\frac{1}{n}\sum_{i=1}^n \boldsymbol{y_i}^T \ln( p(x_i;\theta) ) + \\
    & \quad \quad \quad \quad \quad \alpha \frac{1}{n}\sum_{i=1}^n ( \boldsymbol{y_i} -  q'_i )^T \ln( p(x_i;\theta) ) \\
    & = L_n(\theta) + \alpha \frac{1}{n}\sum_{i=1}^n ( \boldsymbol{y_i} -  q'_i )^T z_i,
\end{align*}
where $z_i \in \mathbb R^K$ is the vector logits for example $x_i$. 
\end{proof}

\subsection{Analytic solution for ALS formula}
\label{app:als}

\begin{lemma}
Let $\alpha \in [0, 1]$, $t \in [\![k]\!]$, and $g \in \mathbb R^k$. The
general solution of the problem
\begin{eqnarray}
  \underset{q \in \Delta_k,\, q^{(t)} \ge 1-\alpha}{\text{argmax}} q^T g
\end{eqnarray}
is $q^*=(1-\alpha) \delta_t + \alpha\bar{q}$, where $\bar{q}$ is any solution to
the problem with $1-\alpha=0$, namely $\bar{q} \in \text{argmax}_{q \in \Delta_k}q^Tg$
\label{thm:bumbednash}
\end{lemma}

\begin{proof}
Consider the invertible change of variable $q = h(\bar{q}):=(1-\alpha)\delta_1 +
\alpha\bar{q}$ which maps the simplex $\Delta_k$ unto itself, with inverse
$\bar{q} = h^{-1}(x)=\alpha^{-1}(x-(1-\alpha)\delta_1)$.

It follows, that
\begin{align*}
  \min_{q \in \Delta_k \mid q_1 \ge 1-\alpha}q^Tb &=\min_{\bar{q} \in \Delta_k \mid \bar{q}_1 \ge 0}\left( (1-\alpha)\delta_1 + \alpha \bar{q}\right)^Tb \\
  &= \min_{\bar{q} \in \Delta_k}((1-\alpha)\delta_1+\alpha\bar{q})^Tb
\end{align*}
which is attained by $$\bar{q}^* \in \text{argmin}_{\bar{q} \in
  \Delta_k}\bar{q}^Tb=\text{ConvHull}(\text{argmin}_{j=1}^k b_j),$$ yielding $ q^*=(1-\alpha)\delta_1 + \alpha\bar{q}^*$.
\end{proof}

\subsection{Proof of Lemma \ref{lemma:fading_gaussian}}
\label{app:fading_gaussian}
For a general covariance matrix in the fading Gaussian model, the best possible adversarial robustness accuracy is
$$
\max_{w \in \mathbb R^d}\text{acc}_\varepsilon(f_w)=\Psi \left( \frac{ w^T \mu - \varepsilon ||w||_1}{\|w\|_\Sigma} \right),
$$
where $\|w\|_\Sigma := \sqrt{w^T\Sigma w}$. Since the Gaussian CDF $\Phi$ is an increasing function, and the objective function in the above problem is $1$-homogeneous in $w$,
we are led to consider problems of the form
\begin{eqnarray}
\alpha &:=\min_{\|w\|_\Sigma \le 1} \varepsilon\|w\|_1 - w^Ta,
\end{eqnarray}
where $a \in \mathbb R^d$ and $\Sigma$ be a positive definite matrix of size $n$.
Of course, the solution value might not be analytically expressible in general, but there is some hope, when the matrix $\Sigma$ is diagonal. That notwithstanding, using the dual representation of the $\ell_1$-norm, one has
\begin{eqnarray}
\begin{split}
\alpha &= \min_{\|w\|_\Sigma \le 1}\max_{\|z\|_\infty \le \varepsilon}z^Tw-w^T a\\
&=\max_{\|z\|_\infty \le \varepsilon}\min_{\|w\|_\Sigma \le 1}w^T(z-a)\\
&=\max_{\|z\|_\infty \le \varepsilon}-\left(\max_{\|w\|_\Sigma \le 1}-w^T(z-a)\right)\\
&=\max_{\|z\|_\infty \le \varepsilon}-\left(\max_{\|\tilde{w}\|_2 \le 1}-\tilde{w}^T\Sigma^{-1}(z-a)\right)\\
&=\max_{\|z\|_\infty \le \varepsilon}-\|z-a\|_{\Sigma^{-1}}
=-\min_{\|z\|_\infty \le \varepsilon}\|z-a\|_{\Sigma^{-1}},
\end{split}
\label{eq:alpha}
\end{eqnarray}
where we have used \emph{Sion's minimax theorem} to interchange min and max in the first line, and we have introduced the auxiliary variable $\tilde{w}:=\Sigma^{-1/2}w$ in the fourth line. We note that given a value for the dual variable $z$, the optimal value of the primal variable $w$ is
\begin{eqnarray}
  w \propto \frac{\Sigma^{-1}(a-z)}{\|\Sigma^{-1}(a-z)\|_2}
  \label{eq:wopt}
\end{eqnarray}

The above expression ~\eqref{eq:alpha} for the optimal objective value $\alpha$ is unlikely to be computable analytically in general, due to the non-separability of the objective (even though the constraint is perfectly separable as a product of 1D constraints).
In any case, it follows from the above display that $\alpha \le 0$, with equality iff $\|a\|_\infty \le \varepsilon$.
\qed

\paragraph{Exact formula for diagonal $\Sigma$.}
In the special case where $\Sigma=\text{diag}(\sigma_1,\ldots,\sigma_2)$, the square of the optimal objective value $\alpha^2$ can be separated as
\begin{eqnarray*}
\begin{split}
\alpha \le 0,\;\alpha^2 &= \sum_{i=1}^d\min_{|z_i| \le \varepsilon}\sigma_i^{-2}(z_i-a_i)^2\\
&= \sum_{i=1}^d\sigma_i^{-2}\begin{cases}(a_i+\varepsilon)^2,&\mbox{ if }a_i \le -\varepsilon,\\ 0,&\mbox{ if }-\varepsilon < a_i \le \varepsilon,\\ (a_i-\varepsilon)^2,&\mbox{ if }a_i > \varepsilon,\end{cases}\\
&= \sum_{j=1}^d((|a_j|-\varepsilon)_+)^2, 
\end{split}
\end{eqnarray*}
which is indeed an analytical formula, albeit a very "hairy" one. 

By the ways, the optimium is attained at
\begin{eqnarray}
  \begin{split}
  z_i &= \begin{cases}
    -\varepsilon,&\mbox{ if }a_i \le -\varepsilon,\\ a_i,&\mbox{ if }-\varepsilon  < a_i \le \varepsilon,\\
    \varepsilon,&\mbox{ if }a_i > \varepsilon,    
\end{cases}\\
&=a_i-\sign(a_i)(|a_i|-\varepsilon)_+
  \end{split}
\end{eqnarray}
Plugging this into \eqref{eq:wopt} yields the optimal weights
\begin{eqnarray}
w_i \propto \sigma^{-2}_j\sign(a_j)(|a_j|-\varepsilon)_+.
\end{eqnarray}

\paragraph{Upper and lower bounds for general $\Sigma$.}
Let $\sigma_d \ge \ldots \ge \sigma_2 \ge \sigma_1 > 0$ be the eigenvalues of $\Sigma$. Then, one has
$$
\sigma_d^{-1}\|w-a\|_2 \le \|w-a\|_{\Sigma^{-1}} \le \sigma_1^{-1}\|w-a\|_2.
$$
Thus one has the bounds
$-\sqrt{\gamma/\sigma_1} \le \alpha \le -\sqrt{\gamma/\sigma_d}$,
where $\gamma := \sum_{i=1}^d\begin{cases}(a_i+\varepsilon)^2,&\mbox{ if }a_i \le -\varepsilon,\\ 0,&\mbox{ if }-\varepsilon < a_i \le \varepsilon,\\ (a_i-\varepsilon)^2,&\mbox{ if }a_i > \varepsilon.\end{cases}$

Moreover, if $\|a\|_\infty > \varepsilon$, then it isn't hard to see that $\gamma \le d(\|a\|_\infty-\varepsilon)$ (see details below), from where $\alpha \ge -\sqrt{\gamma/\sigma_1} \ge -\sqrt{d/\sigma_1}(\|a\|_\infty-1)$, which corresponds to a bound which has been observed by someone in the comments. However, by construction, this bound is potentially very loose.

Note that
$$
a_i \le -\varepsilon \implies (a_i + \varepsilon)^2 \le (\|a\|_\infty - \varepsilon)^2,
$$
and similarly
$$
a_1 > 1 \implies 0 \le (a_i - \varepsilon)^2 \le (\|a\|_\infty-\varepsilon)^2.
$$
Thus $\gamma \le d(\|a\|_\infty-\varepsilon)$.
\qed

\subsection{Experiments: numerical results}
\label{app:tables}

The following tables show the adversarial accuracy for different model set-ups, defenses and attacks. In each table, we have the adversarial accuracies for one attack (or the standard accuracies in Table \ref{tab:accuarcy}). Accuracies for LS-regularized models are presented for three different choices of $\alpha: 0.005, 0.1$ and $0.4$. For FGSM and BIM (Tables \ref{tab:fgsm} and \ref{tab:bim}), we chose 3 different values of the attack strength $\epsilon: 0.05, 0.2$ and $0.4$.
For example, the adversarial accuracy against FGSM attack with $\epsilon = 0.2$ for the BLS-regularized model with $\alpha = 0.005$ using MNIST LeNet set-up is shown in Table \ref{tab:fgsm} and is equal to $0.838$.

Moreover, we highlighted in color the best accuracy for a set-up and a particular attack (or attack and strength in the case of FGSM and BIM). Each set-up corresponds to one color (e.g. light yellow for MNIST Linear and red for SVHN LeNet). If the best accuracy is less than the accuracy obtained with random predictions (i.e. $0.1$ in all our set-ups), it is not highlighted. For example, the best accuracy against FGSM of strenght $\epsilon=0.05$ in the CIFAR LeNet set-up is equal to $0.160$ and is obtained by both a ALS and BLS-regularized NN with $\alpha=0.1$. Overall, adversarial training is better on FGSM (more colors on the adversarial training lines in Table \ref{tab:fgsm} compared to other defenses), but ALS and BLS are better on other attacks. SBLS is better only against C\&W. In Table \ref{tab:accuarcy}, we see that ALS, BLS and SLS NNs are always better or equivalent to a normal classifier (no regularization, no defense method) in terms of standard accuracy.

\begin{table*}[h!]
\begin{minipage}{.5\textwidth}
  \begin{center}
    \caption{FGSM}
    \label{tab:fgsm}
    \setlength\tabcolsep{2pt}
    \resizebox{0.97\textwidth}{!}{%
    \begin{tabular}{c  c || c | c | c || c | c | c || c | c | c ||}
      \multicolumn{2}{ c ||}{} & \multicolumn{3}{c ||}{$\epsilon = 0.05$} & \multicolumn{3}{c ||}{$\epsilon = 0.2$} & \multicolumn{3}{c ||}{$\epsilon = 0.4$} \\
      \hline
       \multicolumn{2}{ c ||}{$\alpha$ val.} & $ 0.005$ & $0.1$ & $0.4$ & $0.005$ & $0.1$ & $0.4$ & $0.005$ & $0.1$ & $0.4$ \\
      \hline \hline
      \multirow{5}{*}{ALS} & MNIST Linear & 0.832 & 0.870 & 0.857 & 0.507 & 0.526 & 0.489 & 0.450 & 0.465 & 0.432  \\ \cline{2-11}
      & MNIST LeNet & 0.957 & 0.959 & 0.954 & 0.847 & 0.732 & 0.639 & 0.822 & 0.561 & 0.130 \\ \cline{2-11}
      & CIFAR LeNet & 0.098 & 0.160 & 0.002 & 0.069 & \resc{0.147} & 0.002 & 0.067 & \resc{0.151} & 0.001 \\ \cline{2-11}
      & CIFAR ResNet  & 0.245 & 0.365 & 0.433 & 0.099 & 0.125 & 0.123 & / & / & / \\ \cline{2-11}
      & SVHN LeNet & 0.450 & 0.458 & \rese{0.497} & 0.370 & 0.302 & \rese{0.381} & / & / & / \\ \cline{2-11}
      \hline \hline
       \multirow{5}{*}{SLS} & MNIST Linear & 0.818 & 0.848 & 0.840 & 0.532 & 0.470 & 0.412 & \resa{0.506} & 0.437 & 0.384  \\ \cline{2-11}
      & MNIST LeNet  & 0.956 & 0.956 & 0.954 & 0.840 & 0.764 & 0.656 & \resb{0.823} & 0.699 & 0.229 \\ \cline{2-11}
      & CIFAR LeNet & 0.119 & 0.153 & 0.127 & 0.097 & 0.134 & 0.101 & 0.092 & 0.136 & 0.093 \\ \cline{2-11}
      & CIFAR ResNet & 0.254 & 0.345 & 0.389 & 0.098 & 0.115 & 0.116 & / & / & / \\ \cline{2-11}
      & SVHN LeNet & 0.404 & 0.472 & 0.395 & 0.353 & 0.312 & 0.195 & / & / & / \\ \cline{2-11}
       \hline \hline
       \multirow{5}{*}{BLS} & MNIST Linear  & 0.821 & 0.868 & 0.858 & 0.494 & 0.525 & 0.494 & 0.442 & 0.457 & 0.441  \\ \cline{2-11}
      & MNIST LeNet  & 0.956 & 0.958 & 0.954 & 0.838 & 0.741 & 0.642 & 0.812 & 0.616 & 0.131 \\ \cline{2-11}
      & CIFAR LeNet & 0.085 & 0.160 & 0.122 & 0.055 & 0.141 & 0.107 & 0.055 & 0.130 & 0.114 \\ \cline{2-11}
      & CIFAR ResNet & 0.252 & 0.346 & \resd{0.450} & 0.096 & 0.129 & \resd{0.138} & / & / & / \\ \cline{2-11}
      & SVHN LeNet & 0.430 & 0.442 & 0.327 & 0.352 & 0.293 & 0.153 & / & / & / \\ \cline{2-11}
      \hline \hline
      \multirow{5}{*}{SBLS} & MNIST Linear & 0.763 & 0.804 & 0.639 & 0.530 & 0.353 & 0.327 & 0.431 & 0.212 & 0.186  \\ \cline{2-11}
      & MNIST LeNet & 0.955 & 0.956 & 0.929 & 0.855 & 0.695 & 0.491 & \resb{0.836} & 0.279 & 0.141 \\ \cline{2-11}
      & CIFAR LeNet  & 0.103 & 0.143 & 0.136 & 0.061 & 0.098 & 0.068 & 0.058 & 0.085 & 0.053 \\ \cline{2-11}
      & CIFAR ResNet & 0.191 & 0.159 & 0.231 & 0.077 & 0.079 & 0.093 & / & / & / \\ \cline{2-11}
      & SVHN LeNet  & 0.449 & 0.375 & 0.277 & 0.378 & 0.157 & 0.149 & / & / & / \\ \cline{2-11}
      \hline \hline
      \multirow{5}{*}{\shortstack{Normal \\ classifier}} & MNIST Linear &  \multicolumn{3}{ c ||}{0.791} &  \multicolumn{3}{ c ||}{0.003} &  \multicolumn{3}{ c ||}{0.000}  \\ \cline{2-11}
      & MNIST LeNet & \multicolumn{3}{ c ||}{0.956} &  \multicolumn{3}{ c ||}{0.498} &  \multicolumn{3}{ c ||}{0.022} \\ \cline{2-11}
      & CIFAR LeNet &  \multicolumn{3}{ c ||}{0.049} &  \multicolumn{3}{ c ||}{0.002} &  \multicolumn{3}{ c ||}{0.009} \\ \cline{2-11}
      & CIFAR ResNet  &  \multicolumn{3}{ c ||}{0.090} &  \multicolumn{3}{ c ||}{0.052} &  \multicolumn{3}{ c ||}{/} \\ \cline{2-11}
      & SVHN LeNet &  \multicolumn{3}{ c ||}{0.195} &  \multicolumn{3}{ c ||}{0.006} &  \multicolumn{3}{ c ||}{/} \\ \cline{2-11}
      \hline \hline
      \multirow{5}{*}{\shortstack{Adv. \\ training}} & MNIST Linear &  \multicolumn{3}{ c ||}{\resa{0.970}} &  \multicolumn{3}{ c ||}{\resa{0.915}} &  \multicolumn{3}{ c ||}{0.286}  \\ \cline{2-11}
      & MNIST LeNet  &  \multicolumn{3}{ c ||}{\resb{0.975}} &  \multicolumn{3}{ c ||}{\resb{0.952}} &  \multicolumn{3}{ c ||}{\resb{0.823}} \\ \cline{2-11}
      & CIFAR LeNet &  \multicolumn{3}{ c ||}{\resc{0.201}} &  \multicolumn{3}{ c ||}{0.023} &  \multicolumn{3}{ c ||}{0.011} \\ \cline{2-11}
      & CIFAR ResNet &  \multicolumn{3}{ c ||}{/} &  \multicolumn{3}{ c ||}{/} &  \multicolumn{3}{ c ||}{/} \\ \cline{2-11}
      & SVHN LeNet &  \multicolumn{3}{ c ||}{0.475} &  \multicolumn{3}{ c ||}{0.071} &  \multicolumn{3}{ c ||}{0.075} \\ \cline{2-11}
    \end{tabular}
    }
  \end{center}
\end{minipage}%
\begin{minipage}{.5\textwidth}
  \begin{center}
    \caption{BIM}
    \label{tab:bim}
    \setlength\tabcolsep{2pt}
    \resizebox{0.97\textwidth}{!}{%
    \begin{tabular}{c  c || c | c | c || c | c | c || c | c | c ||}
      \multicolumn{2}{ c ||}{} & \multicolumn{3}{c ||}{$\epsilon = 0.05$} & \multicolumn{3}{c ||}{$\epsilon = 0.2$} & \multicolumn{3}{c ||}{$\epsilon = 0.4$} \\
      \hline
       \multicolumn{2}{ c ||}{$\alpha$ val.} & $ 0.005$ & $0.1$ & $0.4$ & $0.005$ & $0.1$ & $0.4$ & $0.005$ & $0.1$ & $0.4$ \\
      \hline \hline
      \multirow{5}{*}{ALS} & MNIST Linear & 0.816 & 0.854 & 0.845 & 0.429 & 0.485 & 0.456 & 0.420 & 0.472 & 0.435  \\ \cline{2-11}
      & MNIST LeNet & 0.947 & 0.946 & 0.945 & 0.813 & 0.614 & 0.606 & 0.811 & 0.560 & 0.484 \\ \cline{2-11}
      & CIFAR LeNet & 0.064 & 0.140 & 0.000 & 0.055 & \resc{0.137} & 0.000 & 0.054 & \resc{0.136} & 0.000 \\ \cline{2-11}
      & CIFAR ResNet  & 0.102 & 0.167 & 0.248 & 0.046 & 0.076 & 0.074 & / & / & / \\ \cline{2-11}
      & SVHN LeNet & 0.403 & 0.408 & \rese{0.470} & 0.386 & 0.324 & \rese{0.435} & / & / & / \\ \cline{2-11}
      \hline \hline
       \multirow{5}{*}{SLS} & MNIST Linear & 0.801 & 0.829 & 0.826 & 0.495 & 0.447 & 0.398 & 0.492 & 0.439 & 0.391  \\ \cline{2-11}
      & MNIST LeNet  & 0.946 & 0.942 & 0.942 & 0.817 & 0.700 & 0.398 & 0.815 & 0.689 & 0.165 \\ \cline{2-11}
      & CIFAR LeNet & 0.093 & 0.131 & 0.100 & 0.087 & 0.128 & 0.096 & 0.087 & 0.128 & 0.094 \\ \cline{2-11}
      & CIFAR ResNet & 0.135 & 0.142 & 0.230 & 0.051 & 0.035 & 0.046 & / & / & / \\ \cline{2-11}
      & SVHN LeNet & 0.371 & 0.421 & 0.322 & 0.362 & 0.345 & 0.193 & / & / & / \\ \cline{2-11}
       \hline \hline
       \multirow{5}{*}{BLS} & MNIST Linear  & 0.803 & 0.853 & 0.841 & 0.534 & 0.481 & 0.458 & \resa{0.528} & 0.470 & 0.443  \\ \cline{2-11}
      & MNIST LeNet  & 0.946 & 0.944 & 0.947 & 0.835 & 0.645 & 0.532 & \resb{0.834} & 0.603 & 0.389 \\ \cline{2-11}
      & CIFAR LeNet & 0.062 & 0.140 & 0.103 & 0.056 & 0.136 & 0.097 & 0.055 & 0.135 & 0.097 \\ \cline{2-11}
      & CIFAR ResNet & 0.153 & 0.142 & \resd{0.265} & 0.060 & 0.043 & 0.075 & / & / & / \\ \cline{2-11}
      & SVHN LeNet & 0.435 & 0.393 & 0.266 & 0.424 & 0.322 & 0.142 & / & / & / \\ \cline{2-11}
      \hline \hline
      \multirow{5}{*}{SBLS} & MNIST Linear & 0.736 & 0.772 & 0.597 & 0.470 & 0.227 & 0.299 & 0.370 & 0.091 & 0.252  \\ \cline{2-11}
      & MNIST LeNet & 0.945 & 0.945 & 0.914 & 0.841 & 0.359 & 0.239 & 0.808 & 0.126 & 0.072 \\ \cline{2-11}
      & CIFAR LeNet  & 0.074 & 0.114 & 0.090 & 0.057 & 0.077 & 0.036 & 0.055 & 0.068 & 0.028 \\ \cline{2-11}
      & CIFAR ResNet & 0.054 & 0.021 & 0.045 & 0.010 & 0.008 & 0.018 & / & / & / \\ \cline{2-11}
      & SVHN LeNet  & 0.423 & 0.265 & 0.200 & 0.383 & 0.073 & 0.116 & / & / & / \\ \cline{2-11}
      \hline \hline
      \multirow{5}{*}{\shortstack{Normal \\ classifier}} & MNIST Linear &  \multicolumn{3}{ c ||}{0.776} &  \multicolumn{3}{ c ||}{0.001} &  \multicolumn{3}{ c ||}{0.000}  \\ \cline{2-11}
      & MNIST LeNet & \multicolumn{3}{ c ||}{0.946} &  \multicolumn{3}{ c ||}{0.114} &  \multicolumn{3}{ c ||}{0.000} \\ \cline{2-11}
      & CIFAR LeNet &  \multicolumn{3}{ c ||}{0.015} &  \multicolumn{3}{ c ||}{0.000} &  \multicolumn{3}{ c ||}{0.000} \\ \cline{2-11}
      & CIFAR ResNet  &  \multicolumn{3}{ c ||}{0.003} &  \multicolumn{3}{ c ||}{0.000} &  \multicolumn{3}{ c ||}{/} \\ \cline{2-11}
      & SVHN LeNet &  \multicolumn{3}{ c ||}{0.117} &  \multicolumn{3}{ c ||}{0.000} &  \multicolumn{3}{ c ||}{/} \\ \cline{2-11}
      \hline \hline
      \multirow{5}{*}{\shortstack{Adv. \\ training}} & MNIST Linear &  \multicolumn{3}{ c ||}{\resa{0.970}} &  \multicolumn{3}{ c ||}{\resa{0.899}} &  \multicolumn{3}{ c ||}{0.288}  \\ \cline{2-11}
      & MNIST LeNet  &  \multicolumn{3}{ c ||}{\resb{0.974}} &  \multicolumn{3}{ c ||}{\resb{0.940}} &  \multicolumn{3}{ c ||}{0.498} \\ \cline{2-11}
      & CIFAR LeNet &  \multicolumn{3}{ c ||}{\resc{0.168}} &  \multicolumn{3}{ c ||}{0.003} &  \multicolumn{3}{ c ||}{0.000} \\ \cline{2-11}
      & CIFAR ResNet &  \multicolumn{3}{ c ||}{/} &  \multicolumn{3}{ c ||}{/} &  \multicolumn{3}{ c ||}{/} \\ \cline{2-11}
      & SVHN LeNet &  \multicolumn{3}{ c ||}{0.419} &  \multicolumn{3}{ c ||}{0.014} &  \multicolumn{3}{ c ||}{0.000} \\ \cline{2-11}
    \end{tabular}}
  \end{center}
\end{minipage}%
\end{table*}
 
\begin{table*}[!h]
\begin{minipage}{.325\textwidth}
  \begin{center}
    \caption{DeepFool}
    \label{tab:deepfool}
    \setlength\tabcolsep{2pt}
    \resizebox{0.97\textwidth}{!}{%
    \begin{tabular}{c  c || c | c | c ||}
      \multicolumn{2}{ c ||}{} & \multicolumn{3}{c ||}{Adv. acc.} \\
      \hline
       \multicolumn{2}{ c ||}{$\alpha$ val.} & $ 0.005$ & $0.1$ & $0.4$ \\
      \hline \hline
      \multirow{5}{*}{ALS} & MNIST Linear  & 0.037 & 0.092 & 0.341  \\ \cline{2-5}
      & MNIST LeNet & 0.047 & 0.299 & 0.529 \\ \cline{2-5}
      & CIFAR LeNet  & 0.005 & 0.006 & 0.001 \\ \cline{2-5}
      & CIFAR ResNet & 0.040 & 0.030 & 0.33 \\ \cline{2-5}
      & SVHN LeNet & 0.020 & 0.010 & 0.008 \\ \cline{2-5}
      \hline \hline
       \multirow{5}{*}{SLS} & MNIST Linear & 0.036 & 0.080 & \resa{0.683}  \\ \cline{2-5}
      & MNIST LeNet & 0.035 & 0.183 & 0.380 \\ \cline{2-5}
      & CIFAR LeNet  & 0.006 & 0.007 & 0.008 \\ \cline{2-5}
      & CIFAR ResNet  & 0.073 & 0.049 & 0.057 \\ \cline{2-5}
      & SVHN LeNet  & 0.020 & 0.007 & 0.007 \\ \cline{2-5}
       \hline \hline
       \multirow{5}{*}{BLS} & MNIST Linear & 0.033 & 0.088 & 0.414  \\ \cline{2-5}
      & MNIST LeNet  & 0.031 & 0.314 & 0.500 \\ \cline{2-5}
      & CIFAR LeNet & 0.007 & 0.008 & 0.008 \\ \cline{2-5}
      & CIFAR ResNet & 0.051 & 0.059 & 0.032 \\ \cline{2-5}
      & SVHN LeNet  & 0.024 & 0.006 & 0.006 \\ \cline{2-5}
      \hline \hline
       \multirow{5}{*}{SBLS} & MNIST Linear  & 0.035 & 0.152 & 0.516  \\ \cline{2-5}
      & MNIST LeNet  & 0.028 & 0.142 & 0.305 \\ \cline{2-5}
      & CIFAR LeNet  & 0.005 & 0.006 & 0.006 \\ \cline{2-5}
      & CIFAR ResNet  & 0.050 & 0.037 & 0.064 \\ \cline{2-5}
      & SVHN LeNet  & 0.024 & 0.014 & 0.038 \\ \cline{2-5}
      \hline \hline
       \multirow{5}{*}{\shortstack{Normal \\ classifier}} & MNIST Linear &  \multicolumn{3}{ c ||}{0.025}  \\ \cline{2-5}
      & MNIST LeNet & \multicolumn{3}{ c ||}{0.050}  \\ \cline{2-5}
      & CIFAR LeNet & \multicolumn{3}{ c ||}{0.009} \\ \cline{2-5}
      & CIFAR ResNet & \multicolumn{3}{ c ||}{0.036} \\ \cline{2-5}
      & SVHN LeNet & \multicolumn{3}{ c ||}{0.031} \\ \cline{2-5}
      \hline \hline
       \multirow{5}{*}{\shortstack{Adv. \\ training}} & MNIST Linear & \multicolumn{3}{ c ||}{0.07}  \\ \cline{2-5}
      & MNIST LeNet  & \multicolumn{3}{ c ||}{\resb{0.571}} \\ \cline{2-5}
      & CIFAR LeNet  & \multicolumn{3}{ c ||}{0.008} \\ \cline{2-5}
      & CIFAR ResNet  & \multicolumn{3}{ c ||}{/} \\ \cline{2-5}
      & SVHN LeNet  & \multicolumn{3}{ c ||}{0.030} \\ \cline{2-5}
    \end{tabular}}
  \end{center}
  \end{minipage}
  \begin{minipage}{.325\textwidth}
  \begin{center}
    \caption{CW}
    \label{tab:cw}
    \setlength\tabcolsep{2pt}
    \resizebox{0.97\textwidth}{!}{%
    \begin{tabular}{c  c || c | c | c ||}
      \multicolumn{2}{ c ||}{} & \multicolumn{3}{c ||}{Adv. acc.} \\
      \hline
       \multicolumn{2}{ c ||}{$\alpha$ val.} & $ 0.005$ & $0.1$ & $0.4$ \\
      \hline \hline
      \multirow{5}{*}{ALS} & MNIST Linear  & 0.014 & 0.056 & 0.073  \\ \cline{2-5}
      & MNIST LeNet & 0.020 & 0.042 & 0.084 \\ \cline{2-5}
      & CIFAR LeNet  & 0.000 & 0.002 & 0.000 \\ \cline{2-5}
      & CIFAR ResNet & 0.022 & 0.098 & 0.205 \\ \cline{2-5}
      & SVHN LeNet & 0.010 & 0.039 & 0.056 \\ \cline{2-5}
      \hline \hline
       \multirow{5}{*}{SLS} & MNIST Linear & 0.013 & 0.025 & 0.031  \\ \cline{2-5}
      & MNIST LeNet & 0.022 & 0.046 & 0.058 \\ \cline{2-5}
      & CIFAR LeNet  & 0.000 & 0.004 & 0.000 \\ \cline{2-5}
      & CIFAR ResNet  & 0.029 & 0.087 & 0.159 \\ \cline{2-5}
      & SVHN LeNet  & 0.006 & 0.044 & 0.044 \\ \cline{2-5}
       \hline \hline
       \multirow{5}{*}{BLS} & MNIST Linear & 0.014 & 0.035 & 0.059  \\ \cline{2-5}
      & MNIST LeNet  & 0.018 & 0.046 & 0.062 \\ \cline{2-5}
      & CIFAR LeNet & 0.000 & 0.004 & 0.000 \\ \cline{2-5}
      & CIFAR ResNet & 0.019 & 0.089 & 0.231 \\ \cline{2-5}
      & SVHN LeNet  & 0.006 & 0.028 & 0.045 \\ \cline{2-5}
      \hline \hline
       \multirow{5}{*}{SBLS} & MNIST Linear  & 0.017 & 0.071 & 0.007  \\ \cline{2-5}
      & MNIST LeNet  & 0.016 & 0.128 & \resb{0.474} \\ \cline{2-5}
      & CIFAR LeNet  & 0.002 & 0.000 & 0.014 \\ \cline{2-5}
      & CIFAR ResNet  & 0.063 & \resd{0.173} & 0.106 \\ \cline{2-5}
      & SVHN LeNet  & 0.009 & 0.062 & 0.052 \\ \cline{2-5}
      \hline \hline
       \multirow{5}{*}{\shortstack{Normal \\ classifier}} & MNIST Linear &  \multicolumn{3}{ c ||}{0.015}  \\ \cline{2-5}
      & MNIST LeNet & \multicolumn{3}{ c ||}{0.026}  \\ \cline{2-5}
      & CIFAR LeNet & \multicolumn{3}{ c ||}{0.000} \\ \cline{2-5}
      & CIFAR ResNet & \multicolumn{3}{ c ||}{0.015} \\ \cline{2-5}
      & SVHN LeNet & \multicolumn{3}{ c ||}{0.031} \\ \cline{2-5}
      \hline \hline
       \multirow{5}{*}{\shortstack{Adv. \\ training}} & MNIST Linear & \multicolumn{3}{ c ||}{0.108}  \\ \cline{2-5}
      & MNIST LeNet  & \multicolumn{3}{ c ||}{0.442} \\ \cline{2-5}
      & CIFAR LeNet  & \multicolumn{3}{ c ||}{0.000} \\ \cline{2-5}
      & CIFAR ResNet  & \multicolumn{3}{ c ||}{/} \\ \cline{2-5}
      & SVHN LeNet  & \multicolumn{3}{ c ||}{0.000} \\ \cline{2-5}
    \end{tabular}}
  \end{center}
  \end{minipage}
  \begin{minipage}{.325\textwidth}
   \begin{center}
    \caption{Std. accuracies}
    \label{tab:accuarcy}
    \setlength\tabcolsep{2pt}
    \resizebox{0.97\textwidth}{!}{%
    \begin{tabular}{c  c || c | c | c ||}
      \multicolumn{2}{ c ||}{} & \multicolumn{3}{ c ||}{Std. acc.} \\
      \hline
       \multicolumn{2}{ c ||}{$\alpha$ val.} & $0.005$ & $0.1$ & $0.4$ \\
      \hline \hline
      \multirow{5}{*}{ALS} & MNIST Linear & 0.979 & \resa{0.981} & 0.976  \\ \cline{2-5}
      & MNIST LeNet & \resb{0.990} & \resb{0.990} & 0.989  \\ \cline{2-5}
      & CIFAR LeNet & 0.623 & \resc{0.664} & 0.148 \\ \cline{2-5}
      & CIFAR ResNet & 0.887 & 0.890 & 0.889 \\ \cline{2-5}
      & SVHN LeNet & 0.890 & \rese{0.894} & 0.879 \\ \cline{2-5}
      \hline \hline
       \multirow{5}{*}{SLS} & MNIST Linear & 0.978 & \resa{0.981} & 0.975  \\ \cline{2-5}
      & MNIST LeNet & 0.989 & \resb{0.990} & 0.986  \\ \cline{2-5}
      & CIFAR LeNet & 0.628 & 0.638 & 0.643 \\ \cline{2-5}
      & CIFAR ResNet & 0.885 & 0.894 & 0.895 \\ \cline{2-5}
      & SVHN LeNet & 0.892 & \rese{0.894} & 0.889  \\ \cline{2-5}
       \hline \hline
       \multirow{5}{*}{BLS} & MNIST Linear & 0.978 & \resa{0.981} & 0.976   \\ \cline{2-5}
      & MNIST LeNet & 0.989 & \resb{0.990} & 0.989  \\ \cline{2-5}
      & CIFAR LeNet & 0.639 & 0.651 & 0.622  \\ \cline{2-5}
      & CIFAR ResNet & 0.888 & 0.889 & \resd{0.897}  \\ \cline{2-5}
      & SVHN LeNet & 0.891 & 0.890 & 0.841 \\ \cline{2-5}
      \hline \hline
      \multirow{5}{*}{SBLS} & MNIST Linear & 0.977 & 0.977 & 0.949  \\ \cline{2-5}
      & MNIST LeNet & 0.989 & 0.988 & 0.975  \\ \cline{2-5}
      & CIFAR LeNet & 0.628 & 0.619 & 0.572  \\ \cline{2-5}
      & CIFAR ResNet & 0.883 & 0.881 & 0.840  \\ \cline{2-5}
      & SVHN LeNet & 0.886 & 0.883 & 0.816 \\ \cline{2-5}
      \hline \hline
      \multirow{5}{*}{\shortstack{Normal \\ classifier}} & MNIST Linear & \multicolumn{3}{ c ||}{0.980}  \\ \cline{2-5}
      & MNIST LeNet & \multicolumn{3}{ c ||}{\resb{0.990}} \\ \cline{2-5}
      & CIFAR LeNet & \multicolumn{3}{ c ||}{0.635} \\ \cline{2-5}
      & CIFAR ResNet & \multicolumn{3}{ c ||}{0.886}  \\ \cline{2-5}
      & SVHN LeNet & \multicolumn{3}{ c ||}{0.884}  \\ \cline{2-5}
      \hline \hline
      \multirow{5}{*}{\shortstack{Adv. \\ training}} & MNIST Linear & \multicolumn{3}{ c ||}{\resa{0.981}}  \\ \cline{2-5}
      & MNIST LeNet & \multicolumn{3}{ c ||}{0.982}  \\ \cline{2-5}
      & CIFAR LeNet & \multicolumn{3}{ c ||}{0.505}  \\ \cline{2-5}
      & CIFAR ResNet & \multicolumn{3}{ c ||}{/}  \\ \cline{2-5}
      & SVHN LeNet & \multicolumn{3}{ c ||}{0.831} \\ \cline{2-5}
    \end{tabular}}
  \end{center}
  \end{minipage}
\end{table*}

\end{document}